\documentclass{article}

\usepackage{microtype}
\usepackage{graphicx}
\usepackage{subfigure}
\usepackage{booktabs} 
\usepackage{amsmath}
\usepackage{amsfonts}
\usepackage{amssymb}
\usepackage{amsthm}
\usepackage{physics}
\usepackage{xcolor}
\usepackage{multicol}

\usepackage{hyperref}

\usepackage[accepted]{icml2021}

\usepackage{general-private}
\usepackage{special-private}
\usepackage{thm-config}

\icmltitlerunning{Generalization Error Bound for Hyperbolic Ordinal Embedding}

\begin{document}

\twocolumn[
\icmltitle{Generalization Error Bound for Hyperbolic Ordinal Embedding}




\begin{icmlauthorlist}
\icmlauthor{Atsushi Suzuki}{one}
\icmlauthor{Atsushi Nitanda}{two}
\icmlauthor{Jing Wang}{one}
\icmlauthor{Linchuan Xu}{three}
\icmlauthor{Marc Cavazza}{one}
\icmlauthor{Kenji Yamanishi}{four}
\end{icmlauthorlist}

\icmlaffiliation{one}{School of Computing and Mathematical Sciences,  Faculty of Liberal Arts and Sciences, University of Greenwich, United Kingdom.}
\icmlaffiliation{two}{Department of Artificial Intelligence, Faculty of Computer Science and Systems Engineering, Kyushu Institute of Technology, Japan.}
\icmlaffiliation{three}{Department of Computing, The Hong Kong Polytechnic University, Hong Kong.}
\icmlaffiliation{four}{Department of Mathematical Informatics, Graduate School of Information Science and Technology, The University of Tokyo, Japan}
\icmlcorrespondingauthor{Jing Wang}{jing.wang@greenwich.ac.uk}

\icmlkeywords{Machine Learning, ICML}

\vskip 0.3in
]



\printAffiliationsAndNotice{}  

\begin{abstract}
Hyperbolic ordinal embedding (HOE) represents entities as points in hyperbolic space so that they agree as well as possible with given constraints in the form of entity $\EntityI$ is more similar to entity $\EntityII$ than to entity $\EntityIII$. It has been experimentally shown that HOE can obtain representations of hierarchical data such as a knowledge base and a citation network effectively, owing to hyperbolic space's exponential growth property. However, its theoretical analysis has been limited to ideal noiseless settings, and its generalization error in compensation for hyperbolic space's exponential representation ability has not been guaranteed. The difficulty is that existing generalization error bound derivations for ordinal embedding based on the Gramian matrix do not work in HOE, since hyperbolic space is not inner-product space. In this paper, through our novel characterization of HOE with decomposed Lorentz Gramian matrices, we provide a generalization error bound of HOE for the first time, which is at most exponential with respect to the embedding space's radius. Our comparison between the bounds of HOE and Euclidean ordinal embedding shows that HOE's generalization error is reasonable as a cost for its exponential representation ability.
\end{abstract}

\section{Introduction}
\label{sec:Intro}
Ordinal embedding, also known as a non-metric multidimensional scaling \citep{shepard1962aanalysis, shepard1962banalysis, kruskal1964multidimensional, kruskal1964nonmetric}, aims to represent entities as points in a metric space so that they are as consistent as possible with given ordinal data in the form of ``entity $\EntityI$ is more similar to entity $\EntityII$ than to entity $\EntityIII$.'' Many ordinal embedding methods have been proposed to obtain representations in Euclidean space, which we call Euclidean ordinal embedding (EOE) in this paper, and their effectiveness has been shown in a variety of machine learning areas such as embedding image data and artist data \citep{DBLP:journals/jmlr/AgarwalWCLKB07, DBLP:conf/icml/TamuzLBSK11, DBLP:conf/mlsp/MaatenW12}. However, Euclidean space has a limitation in embedding data with a hierarchical tree-like structure \citep{DBLP:conf/uist/LampingR94, ritter1999self, DBLP:conf/nips/NickelK17} such as a knowledge base and a complex network. This limitation is due to Euclidean space's polynomial growth property, which means that the volume or surface of a ball in Euclidean space grows polynomially with respect to its radius. This Euclidean space's growth speed is significantly slower than embedding hierarchical data such as an $r$-ary tree ($r \ge 2$) requires, which is exponential.
To overcome this limitation, a few recent papers \citep{DBLP:conf/acml/SuzukiWTNY19, DBLP:conf/kdd/TabaghiD20} have proposed ordinal embedding methods using hyperbolic space for hierarchical data, which we call hyperbolic ordinal embedding (HOE) in this paper. In contrast to Euclidean space's polynomial growth property, hyperbolic space has the exponential growth property, that is, the volume of any ball in hyperbolic space grows exponentially with respect to its radius \citep{DBLP:conf/uist/LampingR94, ritter1999self, DBLP:conf/nips/NickelK17}. As a result, we can embed any tree to hyperbolic space with arbitrarily low distortion \citep{DBLP:conf/gd/Sarkar11}.
By leveraging this hyperbolic space's advantage, \citet{DBLP:conf/acml/SuzukiWTNY19} have proposed an HOE method based on Riemannian stochastic gradient descent and achieved effective embedding of hierarchical tree-like data in low-dimensional space. Recently, \citet{DBLP:conf/kdd/TabaghiD20} have solved the hyperbolic distance geometry problem, which includes HOE as a special case, by semi-definite relaxation of the problem and projection operation from Minkowski space to a hyperboloid. 
These two papers have experimentally shown HOE's potential ability to obtain low-dimensional representations effectively for hierarchical tree-like data such as a knowledge base and a citation network.
However, the theoretical guarantee of HOE's performance is limited to ideal noiseless settings \citep{DBLP:conf/acml/SuzukiWTNY19}, and HOE's generalization performance in general noisy settings has not been theoretically guaranteed, although HOE could have much worse generalization error than EOE in compensation for hyperbolic space's exponential growth property and cause overfitting for real data, which are often noisy.

In this paper, we derive the generalization error bound of HOE in general noisy settings under direct conditions on the radius of the embedding space. To the best of our knowledge, this is the first work that derives a generalization error bound for HOE.
Whereas the generalization error of a learning model reflects the volume of its hypothesis space, owing to hyperbolic space's exponential growth property, we cannot expect HOE to have linear or polynomial generalization error with respect to the embedding space's radius, although it is proved for EOE by \citet{DBLP:conf/nips/JainJN16} reflecting Euclidean space's polynomial growth property.
Hence, our objective is to clarify the dependency of the error bound on the embedding space's radius as well as the number of entities and the size of ordinal data.
In this paper, we show that HOE's generalization error is at most exponential with respect to the embedding space's radius. Also, the bound's dependency on the number of entities and the size of ordinal data is the same up to constant factors as that of EOE. Comparing our bound and that of EOE, we see that we can formally obtain HOE's bound by replacing a linear term in EOE's bound with respect to the embedding space's radius by an exponential term. This means that the generalization error bounds of HOE and EOE reflect the volume of embedding space, and our HOE bound is reasonable as a cost for HOE's exponential representation ability.

The difficulty of deriving HOE's generalization error bound is that the technique for EOE's bound of formulating EOE's model and the restriction on its embedding space by the Gramian matrix does not work for HOE. This is because hyperbolic space is not inner-product space and the Gramian matrix does not reflect its metric structure on which the HOE model is constructed. We solve this problem by our novel characterization of HOE model and the restriction on its embedding space by the decomposed Lorentz Gramian matrices. By our approach, we can formulate HOE model as a linear prediction model where the decomposed Lorentz Gramian matrices work as parameters and the restriction on its embedding space as conditions on the norms of these matrices. The resulting formulation enables us to calculate the Rademacher complexity \citep{DBLP:journals/tit/Koltchinskii01, koltchinskii2000rademacher, DBLP:journals/ml/BartlettBL02} of the HOE model, which gives us a tight generalization error bound for linear prediction models \cite{DBLP:conf/nips/KakadeST08}. Combining our Rademacher complexity calculation with existing standard statistical learning theory method \cite{DBLP:journals/jmlr/BartlettM02}, we obtain our HOE's generalization error bound.

\subsection{Our Contributions}
We derive the generalization error bound for HOE for the first time. Our bound is valid under intuitive conditions on the embedding space's radius and the simple uniform distribution assumption on ordinal data. Our results show that HOE's generalization error is at most exponential with respect to embedding space's radius. The bound's dependency on the number of entities and the size of ordinal data is the same up to constant factors as that of EOE.

\subsection{Related Work}
By leveraging hyperbolic space's exponential growth property, many papers have proposed embedding model using hyperbolic space in variety of areas such as interactive visualization \cite{DBLP:conf/uist/LampingR94, DBLP:conf/kdd/WalterR02}, embedding Internet graph \citep{DBLP:journals/ton/ShavittT08, boguna2010sustaining}, routing problem in geographic communication networks \citep{DBLP:conf/infocom/Kleinberg07}, and modeling complex networks \citep{krioukov2010hyperbolic}. Recently, hyperbolic space has also attracted attention in many areas of machine learning such as graph embedding \cite{DBLP:conf/nips/NickelK17, DBLP:conf/icml/GaneaBH18}, metric multi-dimensional scaling \citep{DBLP:conf/icml/SalaSGR18}, neural networks \citep{DBLP:conf/nips/GaneaBH18, DBLP:conf/nips/ChamiYRL19, DBLP:conf/iclr/GulcehreDMRPHBB19}, word embedding \citep{DBLP:conf/iclr/TifreaBG19}, and multi-relational graph embedding \citep{DBLP:conf/nips/BalazevicAH19}, whereas machine learning methods for data in hyperbolic space have also been proposed \citep{DBLP:conf/aistats/ChoD0B19, DBLP:conf/nips/ChamiGCR20}. Ordinal embedding using hyperbolic space has also been proposed recently \citep{DBLP:conf/acml/SuzukiWTNY19, DBLP:conf/kdd/TabaghiD20}, whereas ordinal embedding has been originally studied intensively in Euclidean settings \citep{DBLP:journals/jmlr/AgarwalWCLKB07, DBLP:conf/icml/TamuzLBSK11, DBLP:conf/mlsp/MaatenW12, DBLP:conf/icml/TeradaL14, DBLP:conf/aistats/HashimotoSJ15, DBLP:conf/mlsp/CucuringuW15, DBLP:conf/aaai/MaZXXCLY18, DBLP:conf/icml/AndertonA19, DBLP:conf/aaai/MaXC19}.  
From theoretical aspects, the low distortion property of hyperbolic space for embedding a tree has been discussed in \citep{DBLP:conf/gd/Sarkar11, DBLP:conf/icml/SalaSGR18, DBLP:conf/acml/SuzukiWTNY19} under noiseless conditions. However, to the best of our knowledge, the generalization error in noisy settings of machine learning model using hyperbolic space as embedding space has not been analyzed.

Rademacher complexity \citep{DBLP:journals/tit/Koltchinskii01, koltchinskii2000rademacher, DBLP:journals/ml/BartlettBL02} is one of the key tools to derive an upper bound for the generalization error of a learning model. The upper bound derivation using Rademacher complexity has been studied in \EG \citep{koltchinskii2002empirical, DBLP:journals/jmlr/BartlettM02}. The Rademacher complexity of a linear prediction model under norm restrictions has intensively been studied in \citep{DBLP:conf/nips/KakadeST08}. Recently, \citet{DBLP:conf/nips/JainJN16} calculated an upper bound of the Rademacher complexity of an EOE model. However, the Rademacher complexity of an HOE model has not been evaluated. In this paper, we evaluate the Rademacher complexity of HOE for the first time.

\section{Preliminaries}
\label{sec:Preliminary}
\paragraph{Notation}
In this paper, the symbol $\DefEq$ is used to state that its left hand side is defined by its right hand side. 
We denote by $\Integer, \Integer_{>0}, \Real, \Real_{\ge 0}$ the set of integers, the set of positive integers, the set of real numbers, and the set of non-negative real numbers, respectively. Suppose that $\NAxes, \NEntities \in \Integer_{>0}$. We denote by $, \Real^{\NAxes}, \Real^{\NAxes, \NEntities}, \Sym^{\NEntities, \NEntities}$ the set of $\NAxes$-dimensional real vectors, the set of real matrices with the size of $\NAxes \times \NEntities$, and the set of $\NEntities \times \NEntities$ symmetric matrices, respectively. For $\NEntities \in \Integer$, $[\NEntities]$ denotes the set $\qty{1, 2, \dots, \NEntities}$ of integers. For a matrix $\Mat{A} \in \Real^{\NAxes, \NEntities}$, we denote by $\qty[\Mat{A}]_{\IAxis, \IEntity}$ the element in the $\IAxis$-row and the $\IEntity$-th column and by $\Tr(\Mat{A})$ the trace of $\Mat{A}$. For a vector $\PointVec \in \Real^{\NAxes}$, we denote by $\norm{\PointVec}_{2}$ the $\ell^2$-norm of $\PointVec$, defined by $\norm{\PointVec}_{2} = \sqrt{\PointVec^\Transpose \PointVec}$. 
For matrices $\Mat{A}, \Mat{B} \in \Real^{\NAxes, \NEntities}$, we denote by $\FInProd{\Mat{A}}{\Mat{B}}$ the Frobenius inner-product of $\Mat{A}$ and $\Mat{B}$, defined by $\Tr(\Mat{A}^\Transpose \Mat{B})$.
For symmetric matrices $\Mat{A}, \Mat{B} \in \Sym^{\NEntities, \NEntities}$, we write $\Mat{A} \PosSemiDef \Mat{B}$ if $\Mat{A} - \Mat{B}$ is positive semi-definite.

\subsection{Ordinal Embedding}
\label{sub:OE}
First, we formulate the ordinal embedding. which is of interest in this paper.
Let $\NEntities$ be the number of entities and we identify the set $[\NEntities]$ with the $\NEntities$ entities. We assume that there exists a true dissimilarity measure $\Dsim^{*}: [\NEntities] \times [\NEntities] \to \Real_{\ge 0}$, where $\Dsim^{*} \qty(\EntityI, \EntityII)$ indicates the true dissimilarity between entity $\EntityI$ and entity $\EntityII$. Ordinal data is a set of ordinal comparisons in the form of entity $\EntityI$ is more similar to entity $\EntityII$ than to entity $\EntityIII$, which indicates $\Dsim^{*} \qty(\EntityI, \EntityII) < \Dsim^{*} \qty(\EntityI, \EntityIII)$ if there is no noise in the comparison. In the following, we formulate ordinal comparisons according to the formulation by \citet{DBLP:conf/nips/JainJN16}. The $\ICmp$-th comparison consists of a pair of a triplet $\qty(\EntityI_{\ICmp}, \EntityII_{\ICmp}, \EntityIII_{\ICmp}) \in \TripletSet$ and a label $\Label_{\ICmp} \in \qty{-1, +1}$, where
\begin{equation}
  \label{eqn:TripleSetDef}
  \TripletSet \DefEq \qty{\qty(\EntityI, \EntityII, \EntityIII) \middle| \EntityI, \EntityII, \EntityIII \in [\NEntities], \EntityII < \EntityIII, \EntityIII \ne \EntityI \ne \EntityII}.
\end{equation}
In \eqref{eqn:TripleSetDef}, we put the restriction $\EntityII < \EntityIII$ to keep the uniqueness of the formulation. Note that $\abs{\TripletSet} = \frac{1}{2} \NEntities \qty(\NEntities - 1) \qty(\NEntities - 2)$.
The label indicates the result of the ordinal comparison. Specifically, $\Label_{\ICmp} = -1$ indicates $\EntityI_{\ICmp}$ is closer to $\EntityII_{\ICmp}$ than to $\EntityIII_{\ICmp}$, and $\Label_{\ICmp} = +1$ indicates its converse. If there is no noise in the comparison, $\Label_{\ICmp} = -1$ and $\Label_{\ICmp} = +1$ means $\Dsim^{*} \qty(\EntityI, \EntityII) < \Dsim^{*} \qty(\EntityI, \EntityIII)$ and $\Dsim^{*} \qty(\EntityI, \EntityII) > \Dsim^{*} \qty(\EntityI, \EntityIII)$, respectively. Note that we also consider noisy comparison cases in this paper.
Also, we assume that $\Dsim^{*} \qty(\EntityI, \EntityII) \ne \Dsim^{*} \qty(\EntityI', \EntityII')$ holds for 
any two different pairs $\qty(\EntityI, \EntityII), \qty(\EntityI', \EntityII')$ of different entities, to avoid ambiguity in comparison, as implicitly assumed also in \citep{DBLP:conf/nips/JainJN16}.
Let $\qty(\ReprSpace, \Distance_{\ReprSpace})$ be a metric space, where $\ReprSpace$ is a set $\Distance_{\ReprSpace}: \ReprSpace \times \ReprSpace \to \Real_{\ge 0}$ is a distance function on $\ReprSpace$. The objective of ordinal embedding in $\ReprSpace$ is to get representations $\Repr_{1}, \Repr_{2}, \dots, \Repr_{\NEntities}$ in some low-dimensional metric space $\ReprSpace$, such that the representations are consistent to the true dissimilarity measure $\Dsim^{*}$, where $\Repr_{\EntityI} \in \ReprSpace$ is the representation of entity $\EntityI \in [\NEntities]$. Specifically, ideal representations should satisfy the following:
\begin{equation}
  \label{eqn:ideal}
  \Dsim^{*} \qty(\EntityI, \EntityII) \lessgtr \Dsim^{*} \qty(\EntityI, \EntityIII) \Leftrightarrow \Distance_{\ReprSpace} \qty(\Repr_{\EntityI}, \Repr_{\EntityII}) \lessgtr \Distance_{\ReprSpace} \qty(\Repr_{\EntityI}, \Repr_{\EntityIII}),
\end{equation}
for a new triplet $\qty(\EntityI, \EntityII, \EntityIII) \in \TripletSet$, which may be unseen in the training data. We call the metric space $\qty(\ReprSpace, \Distance_{\ReprSpace})$ used in ordinal embedding the \emph{embedding space}.
As ordinal embedding represents the dissimilarity between two entities by the distance between the two representations, embedding space selection is essential. 
In the next section, we introduce hyperbolic space, which HOE use as the embedding space.

\subsection{Hyperbolic Space}
Hyperbolic space is a metric space, which has been widely used to represent hierarchical data in machine learning areas, owing to its exponential growth property. In this section, we give a formal definition of hyperbolic space. There exist many well-known models of hyperbolic space, such as the hyperboloid model, the Beltrami-Klein model, \Poincare ball model, and \Poincare half-space model \citep[see, \EG][Chapter 3]{lee2018introduction}, which are isometrically isomorphic to each other. In this paper, we mainly work on the hyperboloid model, which formulates hyperbolic space as a submanifold of Minkowski space. The advantage of the hyperboloid model is that we can use the inner product function of Minkowski space for discussion, which plays key role in our novel characterization of HOE. See \EG \citep[Chapter 3]{lee2018introduction} for details.

The $\qty(1 +\NAxes)$-dimensional Minkowski space $\Minkowski^{1, \NAxes} = \qty(\Real^{\NAxes}, \MInProd{\cdot}{\cdot})$, where $\MInProd{\cdot}{\cdot}: \Minkowski^{1, \NAxes} \times \Minkowski^{1, \NAxes} \to \Real$ is the Lorentz inner-product function defined by 
\begin{equation}
\begin{split}
& \MInProd{\mqty[\Point_{0} & \Point_{1} & \dots & \Point_{\NAxes}]^{\Transpose}}{\mqty[\Point'_{0} & \Point'_{1} & \dots & \Point'_{\NAxes}]^{\Transpose}}
\\
& \DefEq
- \Point_{0} \Point'_{0} + \sum_{\IAxis=1}^{\NAxes} \Point_{\IAxis} \Point'_{\IAxis},
\end{split}
\end{equation}
is the pseudo Euclidean vector space with signature $\qty(1, \NAxes)$, that is, the $\qty(1 + \NAxes)$-dimensional vector space equipped with the non-positive-definite bilinear function $\MInProd{\cdot}{\cdot}$. 
The hyperboloid model of $\NAxes$-dimensional hyperbolic space is a Riemannian manifold $\qty(\Loid^{\NAxes}, \Metric_{\Point})$ embedded in $\qty(1 +\NAxes)$-dimensional Minkowski space $\Minkowski^{1, \NAxes}$, where
\begin{equation}
  \Loid^{\NAxes} \DefEq \qty{\PointVec \in \Real^{1 + \NAxes} \middle| \MInProd{\PointVec}{\PointVec} = -1, \Point_{0} > 0.},
\end{equation}
and $\Metric_{\Point}$ is induced by the inclusion $\Inclusion: \Loid^{\NAxes} \to \Minkowski^{1, \NAxes}: \PointVec \mapsto \PointVec$.
The distance function $\Distance_{\Loid^{\NAxes}}: \Loid^{\NAxes} \times \Loid^{\NAxes} \to \Real_{\ge 0}$ is given by 
\begin{equation}
  \Distance_{\Loid^{\NAxes}} \qty(\PointVec, \PointVec') \DefEq \Arcosh(- \MInProd{\PointVec}{\PointVec'}),
\end{equation}
where $\Arcosh$ is the area hyperbolic cosine function, which is the inverse function of the hyperbolic cosine function.
Hyperbolic space has the \emph{exponential growth property} in that the volume and surface area of any ball in hyperbolic space exponentially grows with respect to its radius. 
For example, the circumference of any ball with a radius of $\Radius$ in two-dimensional hyperbolic space is given by $2 \PiUnit \sinh \Radius = \BigO \qty(\exp \Radius)$ in constrast to $2 \PiUnit \Radius$ in two-dimensional Euclidean space. Owing to this property, in graph embedding setting, we can embed any tree with arbitrarily low distortion to hyperbolic space in graph embedding setting \citep{DBLP:conf/gd/Sarkar11}, and in ordinal embedding setting, we can get representations that satisfy all the ordinal constraints generated by any tree \cite{DBLP:conf/acml/SuzukiWTNY19} if the ordinal data are noiseless.



\subsection{HOE and EOE}
\label{sub:HOEEOE}
In this section, we define HOE. We also formulate EOE for later discussion on comparison between HOE and EOE.
HOE and EOE are ordinal embedding using hyperbolic space and Euclidean space, respectively. 
Specifically, HOE is ordinal embedding to obtain representations in $\Loid^{\NAxes}$ that agree as well as possible with \eqref{eqn:ideal} with $\Distance_{\ReprSpace} = \Distance_{\Loid^{\NAxes}}$. 
Likewise, EOE is ordinal embedding to obtain representations in $\Real^{\NAxes}$ that agree as well as possible with \eqref{eqn:ideal} with $\Distance_{\ReprSpace} = \Distance_{\Real^{\NAxes}}$ where $\Distance_{\Real^{\NAxes}}$ is defined by $\Distance_{\Real^{\NAxes}} \qty(\PointVec, \PointVec') \DefEq \norm{\PointVec' - \PointVec}_{2}$.

As an embedding scheme, we mainly focus on minimizing the \emph{empirical risk function} defined below. Let $\Loss: \Real \to \Real_{\ge 0}$ and $\SomeFunc: \Real \to \Real$ be increasing functions and call them the \emph{loss function} and \emph{dissimilarity transformation function}, respectively.
The empirical risk function $\Risk_{\TrainData}^{\ReprSymb}: \qty(\ReprSpace)^{\NEntities} \to \Real_{\ge 0}$ is defined by
\begin{equation}
  \label{eqn:HOEEmpRisk}
  \hat{\Risk}_{\CmpSet}^{\ReprSymb} \qty(\qty(\ReprVec_{\IEntity})_{\IEntity=1}^{\NEntities})
  \DefEq 
  \frac{1}{\NCmps} \sum_{\ICmp=1}^{\NCmps} \Loss \qty(- \Label_{\ICmp} \Hypothesis \qty(\EntityI_{\ICmp}, \EntityII_{\ICmp}, \EntityIII_{\ICmp}; \qty(\ReprVec_{\IEntity})_{\IEntity=1}^{\NEntities})),
\end{equation}
where $\Loss: \Real \to \Real_{\ge 0}$ is an increasing function, and hypothesis function $\Hypothesis \qty(\cdot; \qty(\ReprVec_{\IEntity})_{\IEntity=1}^{\NEntities}): \TripletSet \to \Real$ is defined by
\begin{equation}
\Hypothesis \qty(\EntityI, \EntityII, \EntityIII; \qty(\ReprVec_{\IEntity})_{\IEntity=1}^{\NEntities}) \DefEq \SomeFunc \qty(\Distance \qty(\ReprVec_{\EntityI}, \ReprVec_{\EntityII})) - \SomeFunc \qty(\Distance \qty(\ReprVec_{\EntityI}, \ReprVec_{\EntityIII})),
\end{equation}
for $\qty(\ReprVec_{\IEntity})_{\IEntity=1}^{\NEntities} \in \qty(\ReprSpace)^{\NEntities}$, where $\SomeFunc: \Real \to \Real$ is an increasing function.

For HOE, if we set $\Loss \qty(x) = \max\qty{0, x + 1}$ and $\SomeFunc \qty(x) = x$, the risk function \eqref{eqn:HOEEmpRisk} is reduced to that proposed by \citet{DBLP:conf/acml/SuzukiWTNY19}. In the following discussion, we set $\SomeFunc \qty(x) = \cosh(x)$ for HOE and $\SomeFunc \qty(x) = x^2$ for EOE as in \citep{DBLP:conf/nips/JainJN16}. Whereas we can extend the following discussion for the case where another function is used for $\SomeFunc$, the generalization error bound for that case is worse than that of $\cosh$ if we follow the discussion below.

We also define the \emph{expected risk function}:
\begin{equation}
  \label{eqn:HOEExpRisk}
  \Risk^{\ReprSymb} \qty(\qty(\ReprVec_{\IEntity})_{\IEntity=1}^{\NEntities})
  \DefEq 
  \Expect_{\qty(\EntityI, \EntityII, \EntityIII), \Label} \Loss \qty(- \Label \Hypothesis \qty(\EntityI, \EntityII, \EntityIII; \qty(\ReprVec_{\IEntity})_{\IEntity=1}^{\NEntities})).
\end{equation}

Fix $\BoundedSet \subset \qty(\ReprSpace)^{\NEntities}$, and we define the empirical risk minimizer $\qty(\hat{\ReprVec}_{\IEntity})_{\IEntity=1}^{\NEntities}$ and expected risk minimizer $\qty(\ReprVec^{*}_{\IEntity})_{\IEntity=1}^{\NEntities}$ by
\vspace{-2pt}
\begin{equation}
  \label{eqn:RiskMinimizers}
  \begin{split}
    \qty(\hat{\ReprVec}_{\IEntity})_{\IEntity=1}^{\NEntities} 
    & \DefEq 
    \ArgMin_{\qty(\ReprVec_{\IEntity})_{\IEntity=1}^{\NEntities} \in \BoundedSet} \hat{\Risk}^{\ReprSymb}_{\CmpSet} \qty(\qty(\ReprVec_{\IEntity})_{\IEntity=1}^{\NEntities}), \\
    \qty(\ReprVec^{*}_{\IEntity})_{\IEntity=1}^{\NEntities} 
    & \DefEq 
    \ArgMin_{\qty(\ReprVec_{\IEntity})_{\IEntity=1}^{\NEntities} \in \BoundedSet} \Risk^{\ReprSymb} \qty(\qty(\ReprVec_{\IEntity})_{\IEntity=1}^{\NEntities}).
  \end{split}
\end{equation}
Our interest in this paper is the \emph{excess risk} given by
\begin{equation}
  \Risk^{\ReprSymb} \qty(\qty(\hat{\ReprVec}_{\IEntity})_{\IEntity=1}^{\NEntities}) - \Risk^{\ReprSymb} \qty(\qty(\ReprVec^{*}_{\IEntity})_{\IEntity=1}^{\NEntities}),
\end{equation}
which shows the generalization error of ordinal embedding.

\section{Finite Sample Generalization Bound for HOE}
\label{sec:Main}
\subsection{Assumptions on Data Generation}
To discuss the generalization error, we need to determine the distribution of data generation. Similar to \citep{DBLP:conf/nips/JainJN16}, we assume that training data $\qty(\qty(\EntityI_{\ICmp}, \EntityII_{\ICmp}, \EntityIII_{\ICmp}), \Label_{\ICmp})$ are generated independently and identically according to the following distributions. 
\begin{assumption}
We assume that the triplet is generated uniformly, that is, for all $\qty(\EntityI, \EntityII, \EntityIII) \in \TripletSet$,
\begin{equation}
  \begin{split}
    \Probab [\qty(\EntityI_{\ICmp}, \EntityII_{\ICmp}, \EntityIII_{\ICmp}) = \qty(\EntityI, \EntityII, \EntityIII)] = \frac{1}{\qty|\TripletSet|}
  \end{split}
\end{equation}
is valid, and the conditional distribution of the label $\Label_{\ICmp}$ given the triplet is determined by the true dissimilarity between $\EntityI_{\ICmp}$ and $\EntityII_{\ICmp}$, and that between $\EntityI_{\ICmp}$ and $\EntityII_{\ICmp}$ as follows:
\begin{equation}
  \begin{split}
    \Probab \qty[\Label_{\ICmp} = + 1 \middle| \qty(\EntityI_{\ICmp}, \EntityII_{\ICmp}, \EntityIII_{\ICmp}) = \qty(\EntityI, \EntityII, \EntityIII)] = \Link \qty(\Dsim^{*} \qty(\EntityI, \EntityII) - \Dsim^{*} \qty(\EntityI, \EntityIII)),
  \end{split}
\end{equation}
where $\Link: \Real \to [0, 1]$ is a fixed function called the \emph{link function} \citep{DBLP:conf/nips/JainJN16}.
\end{assumption}

\subsection{Restriction on Representation Domain}
\label{sub:restriction}
To derive a finite generalization bound, in general, it is necessary to restrict parameters (in embedding cases, representations) to a bounded domain (\EG linear prediction models \citep{DBLP:journals/jmlr/BartlettM02, DBLP:conf/nips/KakadeST08}, neural networks \citep{DBLP:journals/jmlr/BartlettM02, schmidt2020nonparametric}).
In this section, we discuss our restriction on embedding space.
For the derived generalization bound to be practical, the restriction should be simple and geometrically intuitive.
We put the following simple restrictions on the embedding space with respect to its radius.

\begin{definition}
  \label{def:HOERestriction}
  Let $\ReprVec_{0}: = \mqty[1 & 0 & \dots & 0] \in \Loid^{\NAxes}$. For $\Radius, \CBound \in \Real_{\ge 0}$, we define $\BoundedSet_{\Radius}, \BoundedSet^{\CBound}, \BoundedSet_{\Radius}^{\CBound}, \subset \qty(\Loid^{\NAxes})^{\NEntities}$ by
  \begin{equation}
    \begin{split}
      \BoundedSet_{\Radius} 
      & \DefEq 
      \qty{\qty(\ReprVec_{\IEntity})_{\IEntity=1}^{\NEntities} \middle| \forall \IEntity \in [\NEntities]: \Distance_{\Loid^{\NAxes}} \qty(\ReprVec_{0}, \ReprVec_{\IEntity}) \le \Radius},
      \\
      \BoundedSet^{\CBound} 
      & \DefEq 
      \qty{\qty(\ReprVec_{\IEntity})_{\IEntity=1}^{\NEntities} \middle| \frac{1}{\NEntities} \sum_{\IEntity=1}^{\NEntities} \cosh^{2} \Distance_{\Loid^{\NAxes}} \qty(\ReprVec_{0}, \ReprVec_{\IEntity}) \le \cosh^{2} \CBound},
    \end{split}
  \end{equation}
  and $\BoundedSet_{\Radius}^{\CBound} \DefEq \BoundedSet_{\Radius} \cap \BoundedSet_{\Radius}^{\CBound}$, respectively.
\end{definition}
Note that, since $\BoundedSet_{\Radius} \subset \BoundedSet^{\Radius}$, $\BoundedSet_{\Radius}^{\Radius} = \BoundedSet^{\Radius}$ is valid. We mainly use $\BoundedSet_{\Radius}$ as the embedding space restriction since it is the simplest and most practical, we also show later that we can achieve a lower generalization error bound if we can set a small $\CBound$ for the restriction given by $\BoundedSet_{\Radius}^{\CBound}$.

\subsection{Main Result: Generalization Error}
In this section, we give our main results, the upper bound of HOE's Rademacher complexity, to obtain HOE's generalization error bound.
\begin{theorem}
  \label{thm:HOEBound}
  Assume that $\Loss$ is $\LipConst$-Lipschitz and bounded. 
  Define 
  \begin{equation}
    \label{eqn:HOELossBound}
    \LossBound_{\Loss} \DefEq \sup \Loss \qty(\Hypothesis \qty(\Triplet; \qty(\ReprVec_{\IEntity})_{\IEntity=1}^{\NEntities}), \Label) - \inf \Loss \qty(\Hypothesis \qty(\Triplet; \qty(\ReprVec_{\IEntity})_{\IEntity=1}^{\NEntities}), \Label),
  \end{equation}
  where $\sup$ and $\inf$ are taken over all $\Triplet \in \TripletSet, \Label \in \qty{-1, +1}, \qty(\ReprVec_{\IEntity})_{\IEntity=1}^{\NEntities} \in \BoundedSet_{\Radius}^{\CBound}$.
  Let $\qty(\hat{\ReprVec}_{\IEntity})_{\IEntity=1}^{\NEntities}$ and $\qty(\ReprVec_{\IEntity}^{*})_{\IEntity=1}^{\NEntities}$ be the empirical and expected risk minimizer of HOE defined by \eqref{eqn:RiskMinimizers} with $\BoundedSet = \BoundedSet_{\Radius}^{\CBound}$.
  Then with probability at least 1 - $\delta$ we have that  
  \begin{equation}
    \begin{split}
      & \Risk^{\ReprSymb} \qty(\qty(\hat{\ReprVec}_{\IEntity})_{\IEntity=1}^{\NEntities}) - \Risk^{\ReprSymb} \qty(\qty(\ReprVec^{*}_{\IEntity})_{\IEntity=1}^{\NEntities})
      \\
      & \le
      2 \LipConst_{\Loss} \qty(\cosh^{2} \CBound + \sinh^{2} \CBound) \qty(\sqrt{\frac{2 \qty(\NEntities + 1) \ln \NEntities}{\NCmps}} + \frac{\NEntities \ln \NEntities}{\sqrt{12}\NCmps}) 
      \\
      & \quad + 2 \LossBound_{\Loss} \sqrt{\frac{2 \ln \frac{2}{\delta}}{\NCmps}}.
    \end{split}
  \end{equation}
\end{theorem}
If we only know $\Radius$ and $\LipConst_{\Loss}$, we have the following.
\begin{corollary}
  Suppose that $\qty(\hat{\ReprVec}_{\IEntity})_{\IEntity=1}^{\NEntities}$ and $\qty(\ReprVec_{\IEntity}^{*})_{\IEntity=1}^{\NEntities}$ are the empirical and expected risk minimizer of HOE in $\BoundedSet_{\Radius} = \BoundedSet_{\Radius}^{\Radius}$. Then with probability at least 1 - $\delta$ we have that 
  \begin{equation}
    \begin{split}
      & \Risk^{\ReprSymb} \qty(\qty(\hat{\ReprVec}_{\IEntity})_{\IEntity=1}^{\NEntities}) - \Risk^{\ReprSymb} \qty(\qty(\ReprVec^{*}_{\IEntity})_{\IEntity=1}^{\NEntities})
      \\
      & \le
      2 \LipConst_{\Loss} \qty(\cosh^{2} \Radius + \sinh^{2} \Radius) \qty(\sqrt{\frac{2 \qty(\NEntities + 1) \ln \NEntities}{\NCmps}} + \frac{\NEntities \ln \NEntities}{\sqrt{12}\NCmps})
      \\
      & \quad + 4 \LipConst_{\Loss} \cosh^{2} \qty(2 \Radius) \sqrt{\frac{2 \ln \frac{2}{\delta}}{\NCmps}}
      \\
      & = \BigO \qty(\LipConst_{\Loss} \qty(\exp \Radius)^{2} \qty(\sqrt{\frac{ \NEntities \ln \NEntities}{\NCmps}} + \frac{\NEntities \ln \NEntities}{\NCmps} + \sqrt{\frac{\ln \frac{2}{\delta}}{\NCmps}})).
    \end{split}
  \end{equation}
\end{corollary}
\begin{proof}
We can set $\CBound = \Radius$ in \Thm \ref{thm:HOEBound} since $\BoundedSet_{\Radius}^{\Radius} = \BoundedSet^{\Radius}$.
All we need to determine is $\LossBound_{\Loss}$ in \Cor \ref{cor:ERMBound}. For all $\ReprVec_{\EntityI}, \ReprVec_{\EntityII}$, we have that $\Distance_{\Loid^{\NAxes}} \qty(\ReprVec_{0}, \ReprVec_{\EntityI}) \le \Radius$ and $\Distance_{\Loid^{\NAxes}} \qty(\ReprVec_{0}, \ReprVec_{\EntityII}) \le \Radius$. Since hyperbolic space is a metric space, we have that $\Distance_{\Loid^{\NAxes}} \qty(\ReprVec_{\EntityI}, \ReprVec_{\EntityII}) \le 2 \Radius$, from the triangle inequality. Hence $\LossBound_{\Loss} \le 2 \LipConst_{\Loss} \cosh^{2} \qty(2 \Radius)$, which completes the proof.
\end{proof}
\begin{remark}
The dependency of HOE's generalization error bounds suggests that the ordinal data size $\NCmps$ should be $\BigO \qty(\NEntities \ln \NEntities)$, which is much smaller than the number $\abs{\TripletSet}$ of possible triplets $\BigO \qty(\NEntities^3)$ and even smaller than the number of entity pairs $\BigO \qty(\NEntities^2)$. These dependencies of HOE's bound are the same as those of EOE's bound. We discuss the dependency on $\Radius$ in the next section.
\end{remark}
Before giving our proof of \Lem \ref{lem:HOEComplexity}, we first discuss the difference between HOE and EOE in the next section.

\subsection{Comparison between EOE and HOE}
\label{sub:Discussion}
In this section, we compare our generalization error bound of HOE and EOE's bound derived by \citet{DBLP:conf/nips/JainJN16}.
First, we introduce the Gramian matrix and its norms, before introduce the results by \citet{DBLP:conf/nips/JainJN16}.
The Gramian matrix of the representations $\ReprVec_{1}, \ReprVec_{2}, \dots, \ReprVec_{\NEntities} \in \Real^{\NAxes}$ is defined by
\begin{equation}
\GramMat = \mqty[\ReprVec_{1} & \ReprVec_{2} & \dots & \ReprVec_{\NEntities}]^{\Transpose} \mqty[\ReprVec_{1} & \ReprVec_{2} & \dots &
\ReprVec_{\NEntities}] \in \Real^{\NEntities, \NEntities}.
\end{equation}
We denote by $\norm{\cdot}_{*}$ and $\norm{\cdot}_{\max}$ the nuclear norm and max norm defined by 
\begin{equation}
\norm{\Mat{A}}_{*}
\DefEq
\sum_{\IEntity=1}^{\NEntities} \SVal_{\IEntity} \qty(\Mat{A}),
\quad
\norm{\Mat{A}}_{\max}
\DefEq
\max_{\EntityI, \EntityII \in [\NEntities]} \qty[\Mat{A}]_{\EntityI, \EntityII},
\end{equation}
respectively, where $\SVal_{\IEntity} \qty(\GramMat)$ is the $\IEntity$-th singular value of $\GramMat$.

EOE's restriction on its embedding space is given by the following sets.

\begin{definition}
  \label{def:EOERestriction}
  We define $\EBoundedSet_{\SupBound}, \EBoundedSet^{\NucBound}, \EBoundedSet_{\SupBound}^{\NucBound}, \subset \qty(\Real^{\NAxes})^{\NEntities}$ as follows:
  \begin{equation}
    \begin{split}
      \EBoundedSet_{\SupBound} 
      & \DefEq 
      \qty{\qty(\ReprVec_{\IEntity})_{\IEntity=1}^{\NEntities} \middle| \norm{\GramMat}_{\max} \le \SupBound},
      \\
      \EBoundedSet^{\NucBound} 
      & \DefEq 
      \qty{\qty(\ReprVec_{\IEntity})_{\IEntity=1}^{\NEntities} \middle| \norm{\GramMat}_{*} \le \NucBound},
      \\
      \EBoundedSet_{\SupBound}^{\NucBound}
      & \DefEq
      \EBoundedSet_{\SupBound} \cap \EBoundedSet^{\NucBound}.
    \end{split}
  \end{equation}
\end{definition}

The generalization error bound for EOE is given as follows:
\begin{theorem}[\Thm 1 in \citep{DBLP:conf/nips/JainJN16}]
  \label{thm:EOEBound}
  Assume that $\Loss$ is $\LipConst$-Lipschitz.
  Let $\qty(\hat{\ReprVec}_{\IEntity})_{\IEntity=1}^{\NEntities}$ and $\qty(\ReprVec_{\IEntity}^{*})_{\IEntity=1}^{\NEntities}$ be the empirical and expected risk minimizer of EOE defined by \eqref{eqn:RiskMinimizers} with $\BoundedSet = \EBoundedSet_{\SupBound}^{\NucBound}$.
  Then with probability at least 1 - $\delta$ we have that  
  \begin{equation}
    \label{eqn:EOEBoundThm}
    \begin{split}
      & \Risk^{\ReprSymb} \qty(\qty(\hat{\ReprVec}_{\IEntity})_{\IEntity=1}^{\NEntities}) - \Risk^{\ReprSymb} \qty(\qty(\ReprVec^{*}_{\IEntity})_{\IEntity=1}^{\NEntities})
      \\
      & \le
      12 \sqrt{2} \LipConst \frac{\NucBound}{\NEntities} \qty(\sqrt{\frac{\NEntities \ln \NEntities}{\NCmps}} + \frac{\sqrt{3}}{9}\frac{\NEntities \ln \NEntities}{\NCmps}) + 12 \sqrt{2} \LipConst \SupBound \sqrt{\frac{\ln \frac{2}{\delta}}{\NCmps}}.
    \end{split}
  \end{equation}
\end{theorem}
\Thm \ref{thm:EOEBound} is not easy enough to interpret as its restrictions on the embedding space given by \Def \ref{def:EOERestriction} is not intuitive.
Our following lemma shows that these restrictions are no more than those of the radius of the embedding space.

\begin{lemma}
  \label{lem:EuclideanRestriction}
  Let $\GramMat$ be the Gramian matrix of $\ReprVec_{1}, \ReprVec_{2}, \dots \ReprVec_{\NEntities}$. Then the followings are valid.
  \begin{equation}
    \begin{split}
      \norm{\GramMat}_{*} 
      & = 
      \sum_{\IEntity=1}^{\NEntities} \qty[\Distance_{\Real^{\NAxes}} \qty(\ZeroVec, \ReprVec_{\IEntity})]^{2},
      \\
      \norm{\GramMat}_{\max} 
      & = 
      \max_{\IEntity \in [\NEntities]} \qty[\Distance_{\Real^{\NAxes}} \qty(\ZeroVec, \ReprVec_{\IEntity})]^{2}.
    \end{split}
  \end{equation}
\end{lemma}
\begin{proof}
See Supplementary Materials.
\end{proof}
\begin{corollary}
  \label{cor:EuclideanNorms}
  \begin{equation}
    \norm{\GramMat}_{*} \le \NEntities \norm{\GramMat}_{\max}.
  \end{equation}
\end{corollary}

Applying \Lem \ref{lem:EuclideanRestriction}, we have a simplified version of \Thm \ref{thm:EOEBound}. We define $\BoundedSet_{\Radius}^{\Real^{\NAxes}} \subset \qty(\Real^{\NAxes})^{\NEntities}$ by
\begin{equation}
  \BoundedSet_{\Radius}^{\Real^{\NAxes}}
  \DefEq 
  \qty{\qty(\ReprVec_{\IEntity})_{\IEntity=1}^{\NEntities} \middle| \forall \IEntity \in [\NEntities]: \Distance_{\Real^{\NAxes}} \qty(\ZeroVec, \ReprVec_{\IEntity}) \le \Radius}.
\end{equation}

\begin{corollary}
  \label{cor:EOEBound}
  Suppose that $\qty(\hat{\ReprVec}_{\IEntity})_{\IEntity=1}^{\NEntities}$ and $\qty(\ReprVec_{\IEntity}^{*})_{\IEntity=1}^{\NEntities}$ are the empirical and expected risk minimizer in $\BoundedSet_{\Radius}^{\Real^{\NAxes}}$, then with probability at least 1 - $\delta$ we have that 
  \begin{equation}
    \label{eqn:EOEBoundCor}
    \begin{split}
      & \Risk^{\ReprSymb} \qty(\qty(\hat{\ReprVec}_{\IEntity})_{\IEntity=1}^{\NEntities}) - \Risk^{\ReprSymb} \qty(\qty(\ReprVec^{*}_{\IEntity})_{\IEntity=1}^{\NEntities})
      \\
      & \le
      12 \sqrt{2} \LipConst \Radius^2 \qty(\sqrt{\frac{\NEntities \ln \NEntities}{\NCmps}} + \frac{\sqrt{3}}{9}\frac{\NEntities \ln \NEntities}{\NCmps} + \sqrt{\frac{\ln \frac{2}{\delta}}{\NCmps}})
      \\
      & = \BigO \qty(\LipConst_{\Loss} \Radius^{2} \qty(\sqrt{\frac{ \NEntities \ln \NEntities}{\NCmps}} + \frac{\NEntities \ln \NEntities}{\NCmps} + \sqrt{\frac{\ln \frac{2}{\delta}}{\NCmps}}))
    \end{split}
  \end{equation}
\end{corollary}
\begin{remark}
We can formally obtain our generalization error bound of HOE by replacing $\Radius^{2}$ in \eqref{eqn:EOEBoundCor} by $\qty(\exp \Radius)^{2}$ up to constant factor. This is consistent to the fact that the volume of a ball in Euclidean space and hyperbolic space is polynomial and exponential with respect to its radius. We can say that HOE pays cost of exponential generalization error in compensation for its exponential volume for embedding space.
As in other machine learning models, there is the bias-variance trade-off between HOE and EOE.
The comparison between EOE and HOE bounds shows that HOE's generalization error bound is larger than that of EOE, suggesting that we should not use hyperbolic space if euclidean space can represent the dissimilarity among the entities well. However, it does not mean that HOE's expected loss is always larger than that EOE, because HOE may have lower minimum expected loss $\Risk^{\ReprSymb} \qty(\qty(\ReprVec^{*}_{\IEntity})_{\IEntity=1}^{\NEntities})$ than EOE, as existing theoretical analyses have shown in noiseless settings (\EG \citep{DBLP:conf/gd/Sarkar11}).
In fact, if the true dissimilarity measure $\Dsim^{*}$ is given by the graph distance of a weighted tree, where euclidean space cannot represent their metric structure well, then HOE can perform better than EOE with sufficient number $\NCmps$ of ordinal data as discussed in the following. 
Suppose that $[\NEntities]$ is the set of vertices of a weighted tree and $\Dsim^{*} \qty(\EntityI, \EntityII)$ is the graph distance between $\EntityI \in [\NEntities]$ and $\EntityII \in [\NEntities]$ on the weighted tree.
We consider the ramp loss $\Loss_{\RampSymb}$ defined by 
\begin{equation}
  \Loss_{\RampSymb} \qty(x)
  \DefEq
  \begin{cases}
    0 & \quad \textrm{$x \le -1$.} \\
    x + 1 & \quad \textrm{$-1 \le x \le 0$.} \\
    1 & \quad \textrm{$x \ge 1$.} \\
  \end{cases}
\end{equation}
Suppose that the link function is given by 
\begin{equation}
  \Link \qty(x) = \begin{cases}
    \frac{1}{2} + \alpha & \quad \textrm{if $x > 0$}, \\
    \frac{1}{2} - \alpha & \quad \textrm{if $x < 0$}, \\
  \end{cases}
\end{equation}
where $\alpha \in \Real_{> 0}$.
For any embedding model, the expected loss is equal to or larger than $\frac{1}{2} - \alpha$. On the other hand, by $\qty(1 + \epsilon)$-distortion Delaunay embedding algorithm \citep{DBLP:conf/gd/Sarkar11}, we can construct the following representations.
\begin{lemma}
\label{lem:MarginEmbedding}
Suppose that $[\NEntities]$ is the set of vertices of a weighted tree and $\Dsim^{*}: [\NEntities] \times [\NEntities] \to \Real_{\ge 0}$ is given by its graph distance, and assume that for $\EntityI, \EntityII, \EntityI', \EntityII'$ such that $\EntityI \ne \EntityII$, $\EntityI' \ne \EntityII'$ and $\qty{\EntityI. \EntityII} \ne \qty{\EntityI'. \EntityII'}$, $\Dsim^{*} \qty(\EntityI, \EntityII) \ne \Dsim^{*} \qty(\EntityI', \EntityII')$ is valid. Then there exist representations $\ReprVec_{1}, \ReprVec_{2}, \dots, \ReprVec_{\NEntities} \in \Loid^{2}$ that satisfy $\Distance_{\Loid^{2}} \qty(\ReprVec_{\EntityI}, \ReprVec_{\EntityII}) - \Distance_{\Loid^{2}} \qty(\ReprVec_{\EntityI'}, \ReprVec_{\EntityII'}) > 1$ for all $\EntityI, \EntityII, \EntityI', \EntityII' \in [\NEntities]$ such that $\Dsim^{*} \qty(\EntityI, \EntityII) > \Dsim^{*} \qty(\EntityI', \EntityII')$.
\end{lemma}
\begin{proof}
See Supplementary Materials.
\end{proof}
Hence, if we take $\Radius$ sufficiently large, the minimum expected loss of HOE in $\Loid^{2}$ is $\frac{1}{2} - \alpha$.
On the other hand, EOE cannot achieve this minimum value for some trees (\EG all the trees that have a node with degree larger than or equal to 6). In this case, as HOE's generalization error bound converges to zero as $\NCmps \to \infty$, the expected loss of HOE's empirical risk minimizer can perform better than all representations in EOE.
\end{remark}

\subsection{Limitation}
\label{sub:limitation}
Whereas EOE bound by \citet{DBLP:conf/nips/JainJN16} and our HOE clarify the dependency of EOE and HOE's generalization error on the embedding space's radius $\Radius$, respectively, the dependency on the embedding space's dimension $\NAxes$ could be improved. Both of these bounds do not directly depend on $\NAxes$. This is not consistent to our intuition that using low-dimensional space should give low generalization error.  \citet{DBLP:conf/nips/JainJN16} has discussed the dependency of EOE's bound on $\NAxes$ by substituting $\NucBound = \sqrt{\NAxes} \NEntities \SupBound$, based on the following observation:
\begin{equation}
  \norm{\GramMat}_{*} \le \sqrt{\NAxes} \norm{\GramMat}_\mathrm{F} \le \sqrt{\NAxes} \NEntities \norm{\GramMat}_\mathrm{\infty}.
\end{equation}
However, as we have seen in \Cor \ref{cor:EuclideanNorms}, we can directly prove $\norm{\GramMat}_{*} \le \NEntities \norm{\GramMat}_\mathrm{\infty}$. Hence there is no reason to consider condition $\NucBound = \sqrt{\NAxes} \NEntities \SupBound$.
Deriving a tighter bound in terms of the dimension for general ordinal embedding could be future work. 

\section{Proof Techniques}
\label{sec:Proof}
In this section, we explain our techniques used to prove \Thm \ref{thm:HOEBound}. Firstly, we explain why existing generalization error bound derivation for EOE does not work in HOE. Secondly, we introduce the first key idea to solve the problem, which is the reformulation of HOE's hypothesis function by the Lorentz Gramian matrix. Thirdly, we introduce the second key idea, which is the reformulation of the restriction on HOE's embedding space by our novel characterization of the decomposed Lorentz Gramian matrix. Lastly, we provide the sketch of the proof.
\subsection{Difficulty in HOE and Our Solution}
In this section, we explain the difficulty of deriving the generalization error bound and our solutions to prove \Thm \ref{thm:HOEBound}. We first see the approach for the existing EOE case. \citet{DBLP:conf/nips/JainJN16} succeeded in deriving the generalization error bound of EOE following the procedures below: 
\begin{itemize}
\item Converting the hypothesis function to that of a linear prediction problem in the following form:
\begin{equation}
  \begin{split}
    \label{eqn:EOELossConversion}
    \Hypothesis \qty(\EntityI, \EntityII, \EntityIII; \qty(\ReprVec_{\IEntity})_{\IEntity=1}^{\NEntities})
    & = \FInProd{\GramMat}{\CmpMat_{\EntityI, \EntityII, \EntityIII}},
  \end{split}
\end{equation}
where $\CmpMat_{\EntityI, \EntityII, \EntityIII} \in \Sym^{\NEntities, \NEntities}$ is some matrix determined by $\EntityI, \EntityII, \EntityIII$. Here the Gramian matrix works as the parameter of a linear prediction model. 
\item Calculating the Rademacher complexity \citep[See Supplementary Materials for definition.]{DBLP:journals/tit/Koltchinskii01, koltchinskii2000rademacher, DBLP:journals/ml/BartlettBL02} of EOE's hypothesis functions under the restriction on norms of the Gramian matrix given by \Def \ref{def:EOERestriction}, which we have shown are equivalent to those on the radius of the embedding space (\Lem \ref{lem:EuclideanRestriction}).
\end{itemize}
However, this procedure does not straightforwardly work for HOE, owing to following reasons. The first reason is that as hyperbolic space is not inner-product space, it is difficult to convert HOE's hypothesis function to that of a linear prediction model in the form of \eqref{eqn:EOELossConversion}, as long as we use the Gramian matrix as its parameters.
We solve this problem using the Lorentz Gramian matrix to reformulate HOE's hypothesis function, leveraging the fact that hyperbolic space is a sub-manifold of Minkowski space.
The second reason is that, even though we can reformulate HOE's hypothesis function as that of a linear prediction model, that conversion also converts our simple restrictions given in \Def \ref{def:HOERestriction} to an unknown set of the Lorentz Gramian matrices. To solve this second problem, we give a clear characterization of the restrictions on embedding space's radius as conditions on the decomposed Lorentz Gramian matrices. Our characterization by the decomposed Lorentz Gramian matrix enables us to calculate the Rademacher complexity of HOE's hypothesis functions using similar techniques to that used in \citep{DBLP:conf/nips/JainJN16}. Once we can calculate the Rademacher complexity, we can prove \Thm \ref{thm:HOEBound} by the standard statistical learning technique.

\subsection{Reformulation of HOE Hypothesis Function by Lorentz Gramian Matrix}
Recall that Minkowski space is equipped with the Lorentz inner-product function $\MInProd{\cdot}{\cdot}$.
The Lorentz Gramian matrix $\LGramMat$ of representations $\ReprVec_{0}, \ReprVec_{1}, \dots, \ReprVec_{\NEntities} \in \Loid^{\NAxes}$ is defined by $\qty[\LGramMat]_{\EntityI, \EntityII} = \MInProd{\ReprVec_{\EntityI}}{\ReprVec_{\EntityII}}$.
If $\SomeFunc = \cosh$, we can reformulate HOE's hypothesis function using the Lorentz inner-product function as follows:
\begin{equation}
  \begin{split}
    \Hypothesis \qty(\EntityI, \EntityII, \EntityIII; \qty(\ReprVec_{\IEntity})_{\IEntity=1}^{\NEntities})
    & =
    - \MInProd{\ReprVec_{\EntityI}}{\ReprVec_{\EntityII}} + \MInProd{\ReprVec_{\EntityI}}{\ReprVec_{\EntityIII}}.
  \end{split}
\end{equation}
For $\qty(\EntityI, \EntityII, \EntityIII) \in \TripletSet$, we define $\CmpMat_{\EntityI, \EntityII, \EntityIII} \in \Sym^{\NEntities, \NEntities}$ by
\begin{equation}
  \qty[\CmpMat_{\EntityI, \EntityII, \EntityIII}]_{\IEntity, \IEntity'}
  =
  \begin{cases}
    - \frac{1}{2} & \quad \textrm{if $\qty(\IEntity, \IEntity') = \qty(\EntityI, \EntityII), \qty(\EntityII, \EntityI)$}, \\
    + \frac{1}{2} & \quad \textrm{if $\qty(\IEntity, \IEntity') = \qty(\EntityI, \EntityIII), \qty(\EntityIII, \EntityI)$}, \\
    0 & \quad \textrm{otherwise}. \\
  \end{cases}
\end{equation}
Then we can convert $\Hypothesis \qty(\EntityI, \EntityII, \EntityIII; \qty(\ReprVec_{\IEntity})_{\IEntity=1}^{\NEntities})$ as a linear function of the Lorentz Gramian as follows:
\begin{equation}
  \begin{split}
    \label{eqn:LossConversion}
    \Hypothesis \qty(\EntityI, \EntityII, \EntityIII; \qty(\ReprVec_{\IEntity})_{\IEntity=1}^{\NEntities})
    & = \FInProd{\LGramMat}{\CmpMat_{\EntityI, \EntityII, \EntityIII}}.
  \end{split}
\end{equation}
Using \eqref{eqn:LossConversion}, we redefine the empirical risk function as a function of the Lorentz Gramian as follows:
\begin{equation}
  \label{eqn:HOEGramEmpRisk}
  \Risk_{\CmpSet}^{\GramSymb} \qty(\LGramMat)
  \DefEq 
  \frac{1}{\NCmps} \sum_{\ICmp=1}^{\NCmps} \Loss \qty(\Label_{\EntityI, \EntityII, \EntityIII} \cdot \FInProd{\LGramMat}{\CmpMat_{\EntityI_{\ICmp}, \EntityII_{\ICmp}, \EntityIII_{\ICmp}}}).
\end{equation}
By definition, if $\LGramMat$ is the Lorentz Gramian matrix of $\qty(\ReprVec_{\IEntity})_{\IEntity=1}^{\NEntities}$, $  \Risk_{\CmpSet}^{\GramSymb} \qty(\LGramMat) = \hat{\Risk}_{\CmpSet}^{\ReprSymb} \qty(\qty(\ReprVec_{\IEntity})_{\IEntity=1}^{\NEntities})$ is valid.
According to the above reformulation of HOE risk using the Lorentz Gramian matrix, the risk bound is obtained if the range of values is specified that the Lorentz Gramian matrix can take. In the next section, we discuss the exact range.

\subsection{Decomposition of the Lorentz Gramian Matrix}
In \Sub \ref{sub:restriction}, we have put the restrictions on the radius of HOE's embedding space, and in the previous section, we have reformulated HOE hypothesis function as the function of the Lorentz Gramian matrix. To calculate the Rademacher complexity of the set of the HOE hypothesis functions and derive a generalization error bound, it is necessary to characterize the original restrictions on embedding space given in \Def \ref{def:HOERestriction} by conditions on the Lorentz Gramian matrix. In this section, we give our novel characterization of the restrictions as conditions with respect to the decomposed Lorentz Gramian matrices. The following lemma gives conditions on the Lorentz Gramian matrix that is equivalent to the restrictions in \Sub \ref{sub:restriction}.
\begin{lemma}
\label{lem:GramianDecomposition}
Let $\Radius \in \Real_{\ge 0}$ and $\LGramMat, \LGramMat^{-}, \LGramMat^{+} \in \Sym^{\NEntities, \NEntities}$. Define conditions $\emph{\textbf{(a)-(f)}}$ as follows:
\begin{description}
  \item[(a)] $\begin{cases} \emph{\textbf{(a-) }} \LGramMat^{-} \PosSemiDef \ZeroMat, \\ \emph{\textbf{(a+) }} \LGramMat^{+} \PosSemiDef \ZeroMat, \end{cases}$
  \item[(b)] $\begin{cases} \emph{\textbf{(b-) }} \rank \LGramMat^{-} = 1, \\ \emph{\textbf{(b+) }} \rank \LGramMat^{+} \le \NAxes, \end{cases}$
  \item[(c)] $\forall \IEntity \in [\NEntities]: \qty[\LGramMat^{+} - \LGramMat^{-}]_{\IEntity, \IEntity} = -1$,
  \item[(d)] $\forall \EntityI, \EntityII \in [\NEntities]: \qty[\LGramMat^{+} - \LGramMat^{-}]_{\EntityI, \EntityII} \le -1$,
  \item[(e)] $\begin{cases} \emph{\textbf{(e-) }} \norm{\LGramMat^{-}}_{\max} \le \cosh^{2} \Radius, \\ \emph{\textbf{(e+) }} \norm{\LGramMat^{+}}_{\max} \le \sinh^{2} \Radius, \end{cases}$
  \item[(f)] $\begin{cases} \emph{\textbf{(f-) }} \norm{\LGramMat^{-}}_{*} \le \NEntities \cosh^{2} \CBound, \\ \emph{\textbf{(f+) }} \norm{\LGramMat^{+}}_{*} \le \NEntities \sinh^{2} \CBound, \end{cases}$
\end{description}
where conditions $\emph{\textbf{(a)}}$, $\emph{\textbf{(b)}}$, $\emph{\textbf{(e)}}$, and $\emph{\textbf{(f)}}$ are conditions $\emph{\textbf{(a-)}}$ and $\emph{\textbf{(a+)}}$, $\emph{\textbf{(b-)}}$ and $\emph{\textbf{(b+)}}$, $\emph{\textbf{(e-)}}$ and $\emph{\textbf{(e+)}}$, and $\emph{\textbf{(f-)}}$ and $\emph{\textbf{(f+)}}$, respectively. Then

(i) $\LGramMat$ is the Lorentz Gramian matrix of a series of representations $\qty(\ReprVec_{\IEntity})_{\IEntity=1}^{\NEntities} \in \qty(\Loid^{\NAxes})^{\NEntities}$ if and only if there exist $\LGramMat^{-}, \LGramMat^{+} \in \Sym^{\NEntities, \NEntities}$ such that $\LGramMat = \LGramMat^{+} - \LGramMat^{-}$ and conditions $\emph{\textbf{(a)-(d)}}$ are satisfied.

(ii) $\LGramMat$ is the Lorentz Gramian matrix of a series of representations $\qty(\ReprVec_{\IEntity})_{\IEntity=1}^{\NEntities} \in \BoundedSet_{\Radius}$ if and only if there exist $\LGramMat^{-}, \LGramMat^{+} \in \Sym^{\NEntities, \NEntities}$ such that $\LGramMat = \LGramMat^{+} - \LGramMat^{-}$ and conditions $\emph{\textbf{(a)-(e)}}$ are satisfied.

(iii) $\LGramMat$ is the Lorentz Gramian matrix of a series of representations $\qty(\ReprVec_{\IEntity})_{\IEntity=1}^{\NEntities} \in \BoundedSet^{\CBound}$ if and only if there exist $\LGramMat^{-}, \LGramMat^{+} \in \Sym^{\NEntities, \NEntities}$ such that $\LGramMat = \LGramMat^{+} - \LGramMat^{-}$ and conditions $\emph{\textbf{(a)-(d),(f)}}$ are satisfied.

(iv) $\LGramMat$ is the Lorentz Gramian matrix of a series of representations $\qty(\ReprVec_{\IEntity})_{\IEntity=1}^{\NEntities} \in \BoundedSet_{\Radius}^{\CBound}$ if and only if there exist $\LGramMat^{-}, \LGramMat^{+} \in \Sym^{\NEntities, \NEntities}$ such that $\LGramMat = \LGramMat^{+} - \LGramMat^{-}$ and conditions $\emph{\textbf{(a)-(f)}}$ are satisfied.
\end{lemma}
\begin{proof}
See Supplementary Materials.
\end{proof}
We call the pair $\LGramMat^{-}$ and $\LGramMat^{-}$ the \emph{decomposed Lorentz Gramian matrices}.
\Lem \ref{lem:GramianDecomposition} rephrases the geometric restrictions on hyperbolic space to conditions including those on the max norm and nuclear norm of the decomposed Lorentz Gramian matrices. 
The significance of \Lem \ref{lem:GramianDecomposition} is that this rephrasing enables us to use techniques similar to those in \citep{DBLP:conf/nips/JainJN16}, where they considered the restriction on those norms of the ordinary Gramian matrix.
Note that the statement of \Lem \ref{lem:GramianDecomposition} (i) is equivalent to that of \Prop 1 in \citep{DBLP:conf/kdd/TabaghiD20} and we can regard (ii) and (iii) as extensions of (i). However, their proof of the necessity in (i) is incomplete, for which we give a complete proof. See the remark in Supplementary Materials for details.

\subsection{Proof Sketch of \Thm \ref{thm:HOEBound}}
Lastly, we give a brief sketch of our proof of \Thm \ref{thm:HOEBound}.
By the decomposition of the Lorentz Gramian matrix, we have the following form of HOE's hypothesis function.
\begin{equation}
  \begin{split}
    \label{eqn:LossDecomposition}
    \Hypothesis \qty(\EntityI, \EntityII, \EntityIII; \qty(\ReprVec_{\IEntity})_{\IEntity=1}^{\NEntities})
    & = \FInProd{\LGramMat^{+} - \LGramMat^{-}}{\CmpMat_{\EntityI, \EntityII, \EntityIII}},
  \end{split}
\end{equation}
with constraints on norms of $\LGramMat^{+}$ and $\LGramMat^{-}$ given by \Lem \ref{lem:GramianDecomposition}. This enables us to decompose the Rademacher complexity of HOE's hypothesis function into two terms. These decomposed terms can be evaluated in a similar way to that used in \citep{DBLP:conf/nips/JainJN16}.
See \Sec \ref{sec:Rademacher} in Supplementary Materials for the definition of Rademacher complexity, our upper bound of the Rademacher complexity of HOE's hypothesis functions and its proof, and the complete proof of \Thm \ref{thm:HOEBound}.

\section{Conclusion}
\label{sec:Conclusion}
We have shown that HOE's generalization error is at most exponential with respect to the embedding space's radius $\Radius$. Also, we have seen that the bound's dependency on the number of entities and the size of ordinal data is the same up to constant factors as that of EOE. Comparing our bound and that of EOE, we have seen that we can formally obtain HOE's bound by replacing a linear term in EOE's bound with respect to the embedding space's radius by an exponential term. The generalization error bounds of HOE and EOE reflect the volume of embedding space, and our HOE bound is reasonable as a cost for HOE's exponential representation ability. Our results suggest that although we should not use hyperbolic space where Euclidean space can represent the true dissimilarity well, HOE's generalization performance cost is worth paying if the data has hierarchical tree structure, as we have seen through the tree example. Combined with existing analyses of hyperbolic embedding in noiseless settings, our generalization error analysis in general noisy settings provides a guide for embedding space selection in real applications.  

\bibliography{main}
\bibliographystyle{icml2021}

\clearpage

\appendix
\twocolumn[\section*{\Large Supplementary Materials \\ for Generalization Error Bound for Hyperbolic Ordinal Embedding} \section*{}]

\section{Proof of \Lem \ref{lem:EuclideanRestriction}}
\begin{proof}
(The first equation) Let $\ReprMat \DefEq \mqty[\ReprVec_{1} & \ReprVec_{2} \dots \ReprVec_{\NEntities}] \in \Real^{\NAxes, \NEntities}$, and $\ReprMat = \Mat{U} \Mat{\varSigma} \Mat{V}^\Transpose$ be the singular value decomposition of $\ReprMat$. Regarding the Gramian matrix, we can obtain the singular value decomposition of $\GramMat$ as follows.
\begin{equation}
  \begin{split}
    \GramMat
    & = \ReprMat^\Transpose \ReprMat
    \\
    & = \Mat{V} \Mat{\varSigma}^\Transpose \Mat{U}^\Transpose \Mat{U} \Mat{\varSigma} \Mat{V}^\Transpose
    \\
    & = \Mat{V} \Mat{\varSigma}^{2} \Mat{V}^\Transpose.
  \end{split}
\end{equation}
Hence, we have $\norm{\GramMat}_{*} = \Tr(\Mat{\varSigma}^{2})$.
Therefore, we have that
\begin{equation}
  \begin{split}
    \norm{\GramMat}_{*} 
    & = \Tr(\Mat{\varSigma}^{2})
    \\
    & = \Tr(\Mat{\varSigma}^\Transpose \Mat{\varSigma} \Mat{V}^\Transpose \Mat{V})
    \\
    & = \Tr(\Mat{V} \Mat{\varSigma} \Mat{U}^\Transpose \Mat{U} \Mat{\varSigma} \Mat{V}^\Transpose)
    \\
    & = \Tr(\ReprMat^\Transpose \ReprMat)
    \\
    & = \Tr(\GramMat)
    \\
    & = \sum_{\IEntity = 1}^{\NEntities} \qty[\GramMat]_{\IEntity, \IEntity}
    \\
    & = \sum_{\IEntity = 1}^{\NEntities} \ReprVec_{\IEntity}^\Transpose \ReprVec_{\IEntity}
    \\
    & = \sum_{\IEntity = 1}^{\NEntities} \qty(\Distance_{\Real^\NAxes} \qty(\ZeroVec, \ReprVec_{\IEntity}))^{2},
  \end{split}
\end{equation}
which completes the proof of the first equation.

(The second equation) For any $\EntityI, \EntityII \in [\NEntities]$, we have that
\begin{equation}
  \begin{split}
    \qty[\GramMat]_{\EntityI, \EntityII}
    & = \ReprVec_{\EntityI}^\Transpose \ReprVec_{\EntityII}
    \\
    & \le \norm{\ReprVec_{\EntityI}}_{2} \norm{\ReprVec_{\EntityII}}_{2}
    \\
    & \le \max \qty{\norm{\ReprVec_{\EntityI}}_{2}^{2}, \norm{\ReprVec_{\EntityII}}_{2}^{2}}
    \\
    & \le \max_{\IEntity \in [\NEntities]} \norm{\ReprVec_{\IEntity}}_{2}^{2}
    \\
    & = \max_{\IEntity \in [\NEntities]} \qty[\GramMat]_{\IEntity, \IEntity}
    \\
    & \le \max_{\IEntity \in [\NEntities]} \qty(\Distance_{\Real^\NAxes} \qty(\ZeroVec, \ReprVec_{\IEntity}))^{2},
  \end{split}
\end{equation}
where the first inequality holds from the Cauchy Schwartz inequality. 
Hence, we have $\max_{\EntityI, \EntityII \in [\NEntities]} \qty[\GramMat]_{\EntityI, \EntityII} \le \max_{\IEntity \in [\NEntities]} \qty[\GramMat]_{\IEntity, \IEntity}$. Conversely, obviously, $\max_{\EntityI, \EntityII \in [\NEntities]} \qty[\GramMat]_{\EntityI, \EntityII} \ge \max_{\IEntity \in [\NEntities]} \qty[\GramMat]_{\IEntity, \IEntity}$ is valid. Therefore, we have that $\max_{\EntityI, \EntityII \in [\NEntities]} \qty[\GramMat]_{\EntityI, \EntityII} = \max_{\IEntity \in [\NEntities]} \qty[\GramMat]_{\IEntity, \IEntity}$. Since the left hand side equals $\norm{\GramMat}_{\max}$, and right hand side equals $\max_{\IEntity \in [\NEntities]} \qty(\Distance_{\Real^\NAxes} \qty(\ZeroVec, \ReprVec_{\IEntity}))^{2}$, we have that $\norm{\GramMat}_{\max} = \max_{\IEntity \in [\NEntities]} \qty(\Distance_{\Real^\NAxes} \qty(\ZeroVec, \ReprVec_{\IEntity}))^{2}$, which completes the proof of the second equation. 

\end{proof}

\section{Proof of \Lem \ref{lem:MarginEmbedding}}

\begin{proof}
Define 
\begin{equation}
\begin{split}
\MinWeight 
& \DefEq 
\min \qty{\Dsim^{*} \qty(\EntityI, \EntityII) \middle| \EntityI \ne \EntityII},
\\
\MinMargin 
& \DefEq 
\min \qty{\abs{1 - \frac{\Dsim^{*} \qty(\EntityI, \EntityII)}{\Dsim^{*} \qty(\EntityI', \EntityII')}} \middle|
\begin{split} & \qty(\EntityI, \EntityII), \qty(\EntityI', \EntityII') \in [\NEntities] \times [\NEntities], \\ & \EntityI \ne \EntityII, \EntityI' \ne \EntityII', \qty(\EntityI, \EntityII) \ne \qty(\EntityI', \EntityII')\end{split}}.
\end{split}
\end{equation}
We assume that $\MinWeight, \MinMargin > 0$ holds as in the discussion in \Sub \ref{sub:OE}.
Let $\epsilon \DefEq \frac{1}{3}\MinMargin$.
Let $\nu, \eta_{\max}$ be the constants determined on the weighted graph defined by \citep{DBLP:conf/gd/Sarkar11}.
Let $\tau \DefEq \max \qty{\eta_{\max}, \frac{\nu \qty(1 + \epsilon)}{\MinWeight \epsilon}, \frac{1}{\MinWeight \epsilon}}$. Then by the $\qty(1 + \epsilon)$-distortion algorithm, we can obtain representations $\ReprVec_{1}, \ReprVec_{2}, \dots, \ReprVec_{\NEntities} \in \Loid^{2}$ such that 
\begin{equation}
\qty(1 - \epsilon) \tau \Dsim^{*} \qty(\EntityI, \EntityII) < \Distance_{\Loid^{2}} \qty(\ReprVec_{\EntityI}, \ReprVec_{\EntityII}) \le \qty(1 + \epsilon) \tau \Dsim^{*} \qty(\EntityI, \EntityII),
\end{equation}
for any $\EntityI, \EntityII \in [\NEntities]$.
Here, the following is valid for $\EntityI, \EntityII, \EntityI', \EntityII' \in [\NEntities]$: if $\Dsim^{*} \qty(\EntityI, \EntityII) > \Dsim^{*} \qty(\EntityI', \EntityII')$, then
\begin{equation}
  \begin{split}
    & \Distance_{\Loid^{2}} \qty(\ReprVec_{\EntityI}, \ReprVec_{\EntityII}) - \Distance_{\Loid^{2}} \qty(\ReprVec_{\EntityI'}, \ReprVec_{\EntityII'})
    \\
    & \ge \tau \qty(1 - \epsilon) \Dsim^{*} \qty(\EntityI, \EntityII) - \qty(1 + \epsilon) \Dsim^{*} \qty(\EntityI', \EntityII')
    \\
    & \ge \tau \Dsim^{*} \qty(\EntityI, \EntityII) \qty[\qty(1 - \epsilon) - \qty(1 + \epsilon) \frac{\Dsim^{*} \qty(\EntityI', \EntityII')}{\Dsim^{*} \qty(\EntityI, \EntityII)} ]
    \\
    & \ge \tau \Dsim^{*} \qty(\EntityI, \EntityII) \qty[\qty(1 - \frac{1}{3} \MinMargin) - \qty(1 + \frac{1}{3} \MinMargin) \qty(1 - \MinMargin)]
    \\
    & > \tau \Dsim^{*} \qty(\EntityI, \EntityII) \frac{1}{3} \MinMargin
    \\
    & \ge \tau \MinWeight \epsilon
    \\
    & \ge 1.
  \end{split}
\end{equation}
\end{proof}

\section{Proof of \Lem \ref{lem:HOEComplexity}}
\begin{proof}
The statement (iv) follows from (ii) and (iii). Therefore, we prove (i)-(iii) in the following.

(Sufficiency)
(i) Assume that $\qty(\ReprVec_{\IEntity})_{\IEntity=1}^{\NEntities} \in \qty(\Loid^{\NAxes})^{\NEntities}$ is valid. 
For $\IAxis = 0, 1, \dots, \NAxes$ and $\IEntity = 1, 2, \dots, \NEntities$, we denote the $\IAxis$-th element of $\ReprVec_{\IEntity} \in \Loid^{\NAxes}$ by $\Repr_{\IAxis, \IEntity}$, and for $\IEntity = 1, 2, \dots, \NEntities$, we define $\ReprVec_{\IEntity}^{-} \in \Real^{1}$ and $\ReprVec_{\IEntity}^{+}$ by
\begin{equation}
  \begin{split}
    \ReprVec_{\IEntity}^{-} 
    \DefEq 
    \mqty[\Repr_{0, \IEntity}], 
    \ReprVec_{\IEntity}^{+} 
    \DefEq 
    \mqty[\Repr_{1, \IEntity} & \Repr_{2, \IEntity} & \cdots & \Repr_{\NAxes, \IEntity}]^\Transpose.
  \end{split}
\end{equation}
Also, we define $\ReprMat^{-} \in \Real^{1, \NEntities}$ and $\ReprMat^{+} \in \Real^{\NAxes - 1, \NEntities}$ by
\begin{equation}
  \begin{split}
    \ReprMat^{-} 
    & \DefEq 
    \mqty[\ReprVec_{1}^{-} & \ReprVec_{2}^{-} & \cdots & \ReprVec_{\NEntities}^{-}],
    \\
    \ReprMat^{+} 
    & \DefEq 
    \mqty[\ReprVec_{1}^{+} & \ReprVec_{2}^{+} & \cdots & \ReprVec_{\NEntities}^{+}],
  \end{split}
\end{equation}
respectively.
Define $\LGramMat^{-}, \LGramMat^{+} \in \Real^{\NEntities, \NEntities}$ by $\LGramMat^{-} \DefEq \qty(\ReprMat^{-})^\Transpose \ReprMat^{-}$ and $\LGramMat^{+} \DefEq \qty(\ReprMat^{+})^\Transpose \ReprMat^{+}$, respectively.
For all $\Vec{x} \in \Real^{\NEntities}$, $\Vec{x}^\Transpose \LGramMat^{-} \Vec{x} = \qty(\LGramMat^{-} \Vec{x})^\Transpose \LGramMat^{-} \Vec{x} \ge 0$. Therefore, we have $\LGramMat^{-} \PosSemiDef \ZeroMat$. Likewise, $\LGramMat^{+} \PosSemiDef \ZeroMat$ is valid, and thus we obtain \textbf{(a)}. Because $\ReprMat^{-} \in \Real^{1, \NEntities}$ and $\ReprMat^{+} \in \Real^{\NAxes - 1, \NEntities}$, we have $\rank \LGramMat^{-} = 1$ and $\rank \LGramMat^{+} \le \NAxes$, respectively. As $\ReprVec_{\IEntity} \in \Loid^{\NAxes}$, $\ReprVec_{0, \IEntity} \ge 1$ is valid, $\rank \LGramMat^{-} \ne 0$, and therefore $\rank \LGramMat^{-} = 1$. Thus, we have \textbf{(b)}. 
If $\EntityI, \EntityII \in [\NEntities]$, then the following inequality holds:
\begin{equation}
  \begin{split}
    & \qty[\LGramMat^{+} - \LGramMat^{-}]_{\EntityI, \EntityII}
    \\
    & = \qty(\ReprVec_{\EntityI}^{+})^\Transpose \ReprVec_{\EntityII}^{+} - \qty(\ReprVec_{\EntityI}^{-})^\Transpose \ReprVec_{\EntityII}^{-}
    \\
    & = \qty(\ReprVec_{\EntityI}^{+})^\Transpose \ReprVec_{\EntityII}^{+} - \sqrt{1 + \qty(\ReprVec_{\EntityI}^{+})^\Transpose \qty(\ReprVec_{\EntityI}^{+})} \sqrt{1 + \qty(\ReprVec_{\EntityII}^{+})^\Transpose \qty(\ReprVec_{\EntityII}^{+})} \\
    & = \qty(\ReprVec_{\EntityI}^{+})^\Transpose \ReprVec_{\EntityII}^{+} - \norm{\mqty[1 & \qty(\ReprVec_{\EntityI}^{+})^\Transpose]^\Transpose}_2 \norm{\mqty[1 & \qty(\ReprVec_{\EntityII}^{+})^\Transpose]^\Transpose}_2 \\ 
    & \le \qty(\ReprVec_{\EntityI}^{+})^\Transpose \ReprVec_{\EntityII}^{+} - \mqty[1 & \qty(\ReprVec_{\EntityI}^{+})^\Transpose] \mqty[1 & \qty(\ReprVec_{\EntityI}^{+})^\Transpose]^\Transpose \\
    & = 1,
  \end{split}
\end{equation}
where the inequality comes from the Cauchy Schwarz inequality, and the equality holds if $\EntityI = \EntityII$. These imply \textbf{(c)} and \textbf{(d)}, which completes the proof of the sufficiency in (i). 

(ii) Assume $\qty(\ReprVec_{\IEntity})_{\IEntity=1}^{\NEntities} \in \BoundedSet_{\Radius}$, that is, for all $\IEntity \in [\NEntities]$, $\Distance_{\Loid^{\NAxes}} \qty(\ReprVec_{0}, \ReprVec_{\IEntity}) \le \Radius$ is valid. Since $\BoundedSet_{\Radius} \subset \Loid^{\NAxes}$, conditions \textbf{(a)-(d)} holds true from the above discussion. 
Since $\Distance_{\Loid^{\NAxes}} \qty(\ReprVec_{0}, \ReprVec_{\IEntity}) = \Arcosh \qty(- \MInProd{\ReprVec_{0}}{\ReprVec_{\IEntity}}) = \Arcosh \qty(\Repr_{0, \IEntity}) = \Arcosh \qty(\sqrt{1 + \qty(\ReprVec_{\IEntity}^{+})^\Transpose \ReprVec_{\IEntity}^{+}})$, the followings are valid:
\begin{equation}
\label{eqn:Gram2Hyp}
\begin{split}
  \qty(\ReprVec_{\IEntity}^{-})^\Transpose \ReprVec_{\IEntity}^{-} & = \qty(\Repr_{0, \IEntity})^{2} = \cosh^{2} \Distance_{\Loid^{\NAxes}} \qty(\ReprVec_{0}, \ReprVec_{\IEntity}), \\    
  \qty(\ReprVec_{\IEntity}^{+})^\Transpose \ReprVec_{\IEntity}^{+} & = 1 + \cosh^{2} \Distance_{\Loid^{\NAxes}} \qty(\ReprVec_{0}, \ReprVec_{\IEntity}) = \sinh^{2} \Distance_{\Loid^{\NAxes}} \qty(\ReprVec_{0}, \ReprVec_{\IEntity}). \\
\end{split}    
\end{equation}
Therefore, we have
\begin{equation}
  \begin{split}
  \qty[\LGramMat^{-}]_{\IEntity, \IEntity} 
  & = 
  \qty(\ReprVec_{\IEntity}^{-})^\Transpose \ReprVec_{\IEntity}^{-} \le \cosh^{2} \Radius \\
  \qty[\LGramMat^{+}]_{\IEntity, \IEntity}, 
  & = 
  \qty(\ReprVec_{\IEntity}^{+})^\Transpose \ReprVec_{\IEntity}^{+} \le \sinh^{2} \Radius.  
  \end{split}
\end{equation}
For all $\EntityI, \EntityII \in [\NEntities]$, 
\begin{equation}
  \begin{split}
    \qty[\LGramMat^{-}]_{\EntityI, \EntityII} & = \qty(\ReprVec_{\EntityI}^{-})^\Transpose \ReprVec_{\EntityII}^{-} \\
    & \le \norm{\ReprVec_{\EntityI}^{-}}_2 \norm{\ReprVec_{\EntityII}^{-}}_2 \\
    & \le \max_{\IEntity = \EntityI, \EntityII} \qty{\qty(\ReprVec_{\IEntity}^{-})^\Transpose \ReprVec_{\IEntity}^{-}} \\ 
    & \le \max_{\IEntity \in [\NEntities]} \qty{\qty(\ReprVec_{\IEntity}^{-})^\Transpose \ReprVec_{\IEntity}^{-}} \\
    & \le \cosh^{2} \Radius
  \end{split}
\end{equation}
is valid, where the first inequality is from the Cauchy Schwarz inequality. Thus, we have $\norm{\LGramMat^{-}}_{\max} \le \cosh^{2} \Radius$, and likewise, we also have $\norm{\LGramMat^{+}}_{\max} \le \sinh^{2} \Radius$, which imply condition \textbf{(e)}. We have proved the sufficiency in (ii).

(iii) If $\qty(\ReprVec_{\IEntity})_{\IEntity=1}^{\NEntities} \in \BoundedSet_{\Radius}^{\CBound}$, since $\BoundedSet_{\Radius}^{\CBound} \subset \BoundedSet_{\Radius}$, conditions \textbf{(a)-(e)} follows from the above discussion. Let $\ReprMat^{-} = \Mat{U}^{-} \Mat{\varSigma}^{-} \qty(\Mat{V}^{-})^\Transpose$ be the singular value decomposition of $\ReprMat^{-}$, where $\Mat{U} \in \Real^{1, 1}, \Mat{V} \in \Real^{\NEntities, \NEntities}$ are orthogonal and $\Mat{\varSigma}^{-} \in \Real^{1, \NEntities}$ is diagonal. The singular decomposition of $\LGramMat^{-}$ is given by
\begin{equation}
  \begin{split}
    \LGramMat^{-} & = \Mat{V}^{-} \qty(\Mat{\varSigma}^{-})^\Transpose \qty(\Mat{U}^{-})^\Transpose \Mat{U}^{-} \Mat{\varSigma}^{-} \qty(\Mat{V}^{-})^\Transpose \\ 
    & = \Mat{V}^{-} \qty(\Mat{\varSigma}^{-})^\Transpose \Mat{\varSigma}^{-} \qty(\Mat{V}^{-})^\Transpose,
  \end{split}
\end{equation}
where the diagonal elements of $\qty(\Mat{\varSigma}^{-})^\Transpose \Mat{\varSigma}^{-}$ indicate the singular values of $\LGramMat^{-}$. Hence, $\norm{\LGramMat^{-}}_{*} = \Tr(\qty(\Mat{\varSigma}^{-})^\Transpose \Mat{\varSigma}^{-})$. 
As $\Mat{V}$ is orthogonal, we have 
\begin{equation}
  \begin{split}
    \Tr(\qty(\Mat{\varSigma}^{-})^\Transpose \Mat{\varSigma}^{-}) & = \Tr(\qty(\Mat{\varSigma}^{-})^\Transpose \Mat{\varSigma}^{-} \qty(\Mat{V}^{-})^\Transpose \Mat{V}^{-}) \\ 
    & = \Tr \qty(\Mat{V}^{-} \qty(\Mat{\varSigma}^{-})^\Transpose \Mat{\varSigma}^{-} \qty(\Mat{V}^{-})^\Transpose) \\ 
    & = \Tr(\LGramMat^{-}).      
  \end{split}
\end{equation}
Hence, we get 
\begin{equation}
  \begin{split}
    \norm{\LGramMat^{-}}_{*} & = \Tr(\LGramMat^{-}) 
    \\
    & = \sum_{\IEntity=1}^{\NEntities} \qty(\ReprVec_{\IEntity}^{-})^\Transpose \ReprVec_{\IEntity}^{-} \\ 
    & = \sum_{\IEntity=1}^{\NEntities} \cosh^{2} \Distance_{\Loid^{\NAxes}} \qty(\ReprVec_{0}, \ReprVec_{\IEntity}). 
  \end{split}
\end{equation}
Likewise, we have 
\begin{equation}
  \begin{split}
    \norm{\LGramMat^{+}}_{*} & = \Tr(\LGramMat^{+}) \\
    & = \sum_{\IEntity=1}^{\NEntities} \qty(\ReprVec_{\IEntity}^{+})^\Transpose \ReprVec_{\IEntity}^{+}
    \\ 
    & = \sum_{\IEntity=1}^{\NEntities} \sinh^{2} \Distance_{\Loid^{\NAxes}} \qty(\ReprVec_{0}, \ReprVec_{\IEntity}). 
  \end{split}
\end{equation}
By the definition of $\BoundedSet_{\Radius}^{\CBound}$, we have $\norm{\LGramMat^{+}}_{*} \le \NEntities \sinh^{2} \CBound$ and $\norm{\LGramMat^{-}}_{*} \le \NEntities \cosh^{2} \CBound$, which imply condition \textbf{(f)}. This completes the proof of the sufficiency in (iii).

(Necessity) (i) Assume that conditions \textbf{(a)-(d)} are satisfied. Noting that $\LGramMat^{-} \PosSemiDef \ZeroMat$, let $\LGramMat^{-} = \Mat{\tilde{V}}^{-} \Mat{T}^{-} \qty(\Mat{\tilde{V}}^{-})^\Transpose$ be a singular value decomposition of $\LGramMat^{-}$, where $\Mat{\tilde{V}}^{-} \in \Real^{\NEntities. \NEntities}$ is orthogonal and $\Mat{T}^{-}$ is diagonal, that is, $\qty[\Mat{T}^{-}]_{\EntityI, \EntityII} = 0$ if $\EntityI \ne \EntityII$. Since $\rank \Mat{\LGramMat}^{-} = 1$ and $\Mat{\LGramMat}^{-} \PosSemiDef \ZeroMat$, we can assume that $\qty[\Mat{T}^{-}]_{1, 1} > 0$ and $\qty[\Mat{T}^{-}]_{\IEntity, \IEntity} = 0$ for all $\NEntities = 2, 3, \dots, \NEntities$. Therefore, $[\LGramMat^{-}]_{\EntityI, \EntityII} = \qty[\Mat{\tilde{V}}^{-}]_{\EntityI, 1} \qty[\Mat{T}^{-}]_{1, 1} \qty[\Mat{\tilde{V}}^{-}]_{\EntityI, 1}$, for $\EntityI, \EntityII \in [\NEntities]$. In particular, $[\LGramMat^{-}]_{1, 1} = \qty(\qty[\Mat{\tilde{V}}^{-}]_{1, 1})^2 \qty[\Mat{T}^{-}]_{1, 1}$. As $\LGramMat^{+}$ is positive semi-definite, its diagonal entries are all non-negative. In particular, $\qty[\LGramMat^{+}]_{1, 1} \ge 0$. Since $[\LGramMat^{+}]_{1, 1} - [\LGramMat^{-}]_{1, 1} = -1$ from \textbf{(c)}, we have $\qty[\LGramMat^{-}]_{1, 1} \ge 1$. Hence, we have $\qty[\Mat{\tilde{V}}^{-}]_{1, 1} \ne 0$. Define $\Mat{V}^{-}$ by
\begin{equation}
  \Mat{V}^{-} \DefEq 
  \begin{cases}
  \Mat{V}^{-} & \quad \textrm{if $\qty[\Mat{\tilde{V}}^{-}]_{1, 1} > 0$,} \\
  - \Mat{V}^{-} & \quad \textrm{if $\qty[\Mat{\tilde{V}}^{-}]_{1, 1} < 0$.} \\
  \end{cases}
\end{equation}
Then $\qty[\Mat{V}^{-}]_{1, 1} > 0$ and $\LGramMat^{-} = \Mat{V}^{-} \Mat{T}^{-} \qty(\Mat{V}^{-})^\Transpose$ is valid.
Let $\LGramMat^{+} = \Mat{V}^{+} \Mat{T}^{+} \qty(\Mat{V}^{+})^\Transpose$ be a singular value decomposition of $\LGramMat^{+}$, where $\Mat{\tilde{V}}^{+} \in \Real^{\NEntities. \NEntities}$ is orthogonal and $\Mat{T}^{+}$ is diagonal. Since $\rank \Mat{\LGramMat}^{+} \le \NAxes$ and $\Mat{\LGramMat}^{+} \PosSemiDef \ZeroMat$, we can assume that $\qty[\Mat{T}^{-}]_{\IEntity, \IEntity} = 0$ for all $\IEntity = \NAxes + 1, \NAxes + 2, \dots, \NEntities$. Define $\ReprVec_{\IEntity}^{-} \in \Real^{1}$ and $\ReprVec_{\IEntity}^{+} \in \Real^{\NAxes}$ by 
\begin{equation}
  \label{eqn:ReprConstruct}
  \begin{split}
    \ReprVec_{\IEntity}^{-} & \DefEq \mqty[\Repr_{0, \IEntity}], \\
    \ReprVec_{\IEntity}^{+} & \DefEq \mqty[\Repr_{1, \IEntity} & \Repr_{2, \IEntity} & \cdots & \Repr_{\NAxes, \IEntity}]^\Transpose,     
  \end{split}
\end{equation}
respectively, where
\begin{equation}
  \begin{split}
    \Repr_{\IAxis, \IEntity} =
    \begin{cases}
      \sqrt{\qty[\Mat{T}^{-}]_{1, 1}} \qty[\Mat{V}^{-}]_{\IEntity, 1} & \quad \textrm{if $\IAxis = 0$}, \\ 
      \sqrt{\qty[\Mat{T}^{+}]_{\IAxis, \IAxis}} \qty[\Mat{V}^{+}]_{\IEntity, \IAxis} & \quad \textrm{if $\IAxis = 1, 2, \dots, \NAxes$}. \\ 
    \end{cases}
  \end{split}
\end{equation}
respectively, and define $\ReprMat^{-} \in \Real^{1, \NEntities}$ and $\ReprMat^{+} \in \Real^{\NAxes, \NEntities}$ by 
\begin{equation}
  \begin{split}
    \ReprMat^{-} = \mqty[\ReprVec_{1}^{-} & \ReprVec_{2}^{-} & \cdots & \ReprVec_{\NEntities}^{-}], \\
    \ReprMat^{+} = \mqty[\ReprVec_{1}^{+} & \ReprVec_{2}^{+} & \cdots & \ReprVec_{\NEntities}^{+}], \\
  \end{split}
\end{equation}
respectively. Now, for $\IEntity \in [\NEntities]$, we define $\ReprVec_{\IEntity} \in \Real^{1 + \NAxes}$ by
\begin{equation}
  \ReprVec_{\IEntity} = \mqty[\ReprVec_{\IEntity}^{-} \\ \ReprVec_{\IEntity}^{+}].
\end{equation}
The Lorentz Gramian of $\qty(\ReprVec_{\IEntity})_{\IEntity = 1}^{\NEntities}$ is given by
\begin{equation}
  \begin{split}
    \qty(\ReprMat^{+})^\Transpose \ReprMat^{+} - \qty(\ReprMat^{-})^\Transpose \ReprMat^{-}
    & = \Mat{V}^{+} \Mat{T}^{+} \qty(\Mat{V}^{+})^\Transpose - \Mat{V}^{-} \Mat{T}^{-} \qty(\Mat{V}^{-})^\Transpose
    \\
    & = \LGramMat^{+} - \LGramMat^{-}.
  \end{split}
\end{equation}
In the following, we prove $\ReprVec_{\IEntity} \in \Loid^{\NAxes}$ for all $\IEntity \in [\NEntities]$. Since $\MInProd{\ReprVec_{\IEntity}}{\ReprVec_{\IEntity}} = \qty[\LGramMat]_{\IEntity, \IEntity} = \qty[\LGramMat^{+} - \LGramMat^{-}]_{\IEntity, \IEntity} = -1$, it is sufficient to prove $\Repr_{0. \IEntity} > 0$. From $\qty[\Mat{V}^{-}]_{1, 1} > 0$ and $\Repr_{0, \IEntity} = \sqrt{\qty[\Mat{T}^{-}]_{1, 1}} \qty[\Mat{V}^{-}]_{\IEntity, 1}$, we have $\Repr_{0, 1} > 0$. For general $\IEntity \in [\NEntities]$, the following is valid:
\begin{equation}
  \begin{split}
    \abs{\Repr_{0, 1}} \abs{\Repr_{0, \IEntity}} - \Repr_{0, 1} \Repr_{0, \IEntity}
    & >
    \norm{\ReprVec_{1}^{+}}_{2} \norm{\ReprVec_{\IEntity}^{+}}_{2} - \Repr_{0, 1} \Repr_{0, \IEntity}
    \\
    & \ge \qty(\ReprVec_{1}^{+})^\Transpose \ReprVec_{\IEntity}^{+} - \Repr_{0, 1} \Repr_{0, \IEntity}
    \\
    & = - \MInProd{\ReprVec_{1}}{\ReprVec_{\IEntity}}
    \\
    & = \qty[\LGramMat]_{1, \IEntity}
    \\
    & = \qty[\LGramMat^{+} - \LGramMat^{-}]_{1, \IEntity}
    \\
    & \ge 1
    \\
    & > 0.
  \end{split}
\end{equation}
Therefore, $\Repr_{0, 1}$ and $\Repr_{0, \IEntity}$ must have the same sign. Hence, $\Repr_{0, \IEntity} > 0$, which completes the proof of the necessity in (i).

(ii) Assume that conditions \textbf{(a)-(e)} are satisfied. Define $\qty(\ReprVec_{\IEntity})_{\IEntity=1}^{\NEntities}$ as in \eqref{eqn:ReprConstruct}. Then, since \textbf{(a)-(d)} are satisfied, $\qty(\ReprVec_{\IEntity})_{\IEntity=1}^{\NEntities} \in \qty(\Loid^{\NAxes})^{\NEntities}$ and its Lorentz Gramian is $\LGramMat = \LGramMat^{+} - \LGramMat^{-}$. Thus, it suffices to show that condition \textbf{(e)} implies that $\Distance_{\Loid^{\NAxes}} \qty(\ReprVec_{0}, \ReprVec_{\NEntities}) \le \Radius$ is valid for all $\IEntity \in \NEntities$. Since $\qty[\LGramMat^{-}]_{\IEntity, \IEntity} = \qty(\Repr_{0, \IEntity})^{2} = \cosh^{2} \Distance_{\Loid^{\NAxes}} \qty(\ReprVec_{0}, \ReprVec_{\NEntities})$, condition \textbf{(e)} implies $\cosh^{2} \Distance_{\Loid^{\NAxes}} \qty(\ReprVec_{0}, \ReprVec_{\NEntities}) \le \cosh^{2} \Radius$, which is equivalent to $\Distance_{\Loid^{\NAxes}} \qty(\ReprVec_{0}, \ReprVec_{\NEntities}) \le \Radius$. Hence, we have $\qty(\ReprVec_{\IEntity})_{\IEntity=1}^{\NEntities} \in \BoundedSet_{\Radius}$.

(iii) Assume that conditions \textbf{(a)-(f)} are satisfied. Define $\qty(\ReprVec_{\IEntity})_{\IEntity=1}^{\NEntities}$ as in \eqref{eqn:ReprConstruct}. Then, since \textbf{(a)-(e)} are satisfied, $\qty(\ReprVec_{\IEntity})_{\IEntity=1}^{\NEntities} \in \BoundedSet_{\Radius}$ and its Lorentz Gramian is $\LGramMat = \LGramMat^{+} - \LGramMat^{-}$. Also, condition \textbf{(f)} implies $\frac{1}{\NEntities} \sum_{\IEntity=1}^{\NEntities} \cosh^{2} \Distance_{\Loid^{\NAxes}} \qty(\ReprVec_{0}, \ReprVec_{\IEntity}) \le \cosh^{2} \CBound$, since the left hand side is given by $\frac{1}{\NEntities} \norm{\LGramMat^{-}}_{*}$.
\end{proof}
\begin{remark}
The statement of \Lem \ref{lem:GramianDecomposition} (i) is equivalent to that of \Prop 1 in \citep{DBLP:conf/kdd/TabaghiD20}. However, the proof there does not consider the case where the representations given by the decomposition of the Lorentz Gramian are all in $- \Loid^{\NAxes}$ instead of $\Loid^{\NAxes}$. The above construction of representations solves this problem by forcing $\Repr_{0, \IEntity}$ to be positive.
\end{remark}

\section{Rademacher Complexity and Proof of \Thm \ref{thm:HOEBound}}
\label{sec:Rademacher}
First, we define Rademacher complexity \citep{DBLP:journals/tit/Koltchinskii01, koltchinskii2000rademacher, DBLP:journals/ml/BartlettBL02}, the key tool to obtain HOE's generalization error bound. Let $\InputSp$ be our input space and $\HypothesisSp \subset \qty{\Hypothesis \middle| \Hypothesis: \InputSp \to \Real}$ be our hypothesis space. Let $\NData \in \Integer_{\ge 0}$ be the number of data points, and suppose that data points $\qty(\Input_{1}, \Label_{1}), \qty(\Input_{2}, \Label_{2}), \dots, \qty(\Input_{\NData}, \Label_{\NData}) \in \InputSp \times \qty{-1, +1}$ are independently distributed according to some unknown fixed distribution $\Distribution$. The Rademacher complexity of $\Hypothesis$ is defined as follows:
\begin{definition}
  Let $\Rdm_{1}, \Rdm_{2}, \dots, \Rdm_{\NData}$ be random values such that $\Rdm_{1}, \Rdm_{2}, \dots, \Rdm_{\NData}, \qty(\Input_{1}, \Label_{1}), \qty(\Input_{2}, \Label_{2}), \dots, \qty(\Input_{\NData}, \Label_{\NData})$ are mutually independent and each of $\Rdm_{1}, \Rdm_{2}, \dots, \Rdm_{\NData}$ takes values $\qty{-1, +1}$ with equal probability. The Rademacher complexity $\RdmCmpl_{\NData} \qty(\HypothesisSp)$ is defined by
  \begin{equation}
    \RdmCmpl_{\NData} \qty(\HypothesisSp)
    \DefEq
    \Expect_{\qty(\Input_{\IDatum})_{\IDatum=1}^{\NData}} \Expect_{\qty(\Rdm_{\IDatum})_{\IDatum=1}^{\NData}}
    \qty[\frac{1}{\NData} \sup_{\Hypothesis \in \HypothesisSp} \sum_{\IDatum=1}^{\NData} \Rdm_{\IDatum} \Hypothesis \qty(\Input_{\IDatum})].
  \end{equation}
\end{definition}
We use the following theorem provided by \citet{DBLP:journals/jmlr/BartlettM02} and arranged by \citet{DBLP:conf/nips/KakadeST08}.
\begin{theorem}[\citep{DBLP:journals/jmlr/BartlettM02, DBLP:conf/nips/KakadeST08}]
  \label{thm:ULLN}
  Let $\Loss: \InputSp \times \qty{-1. +1} \to \Real$ be a loss function.
  Define the empirical risk function $\hat{\Risk}_{\DataSet} \qty(\Hypothesis)$ and expected risk function $\Risk \qty(\Hypothesis)$ by
  \begin{equation}
    \begin{split}
      \hat{\Risk}_{\DataSet} \qty(\Hypothesis)
      & \DefEq
      \frac{1}{\NData} \sum_{\IDatum=1}^{\NData} \Loss \qty( \qty(\Hypothesis \qty(\Input_{\IDatum}), \Label_{\IDatum})),
      \\
      \Risk \qty(\Hypothesis)
      & \DefEq
      \Expect_{\Input, \Label} \Loss \qty(\Hypothesis\qty(\Input), \Label)).
    \end{split}
  \end{equation}
  Assume that $\Loss \qty(\cdot, -1)$ and $\Loss \qty(\cdot, +1)$ are $\LipConst_{\Loss}$-Lipschitz and bounded. Define 
  \begin{equation}
    \label{eqn:LossBound}
    \LossBound_{\Loss} \DefEq \sup_{\mathclap{\substack{\Input \in \InputSp, \\ \Label \in \qty{-1, +1}, \\ \Hypothesis \in \HypothesisSp}}} \Loss \qty(\Hypothesis\qty(\Input), \Label) - \inf_{\mathclap{\substack{\Input \in \InputSp, \\ \Label \in \qty{-1, +1}, \\ \Hypothesis \in \HypothesisSp}}} \Loss \qty(\Hypothesis\qty(\Input), \Label).  
  \end{equation}
  Then for any $\delta \in \Real_{>0}$ and with probability at least $1 - \delta$ simultaneously for all $\Hypothesis \in \HypothesisSp$ we have that
  \begin{equation}
    \begin{split}
      \Risk \qty(\Hypothesis) - \hat{\Risk}_{\DataSet} \qty(\Hypothesis)
      \le
      2 \RdmCmpl_{\NData} \qty(\HypothesisSp) + \LossBound_{\Loss} \sqrt{\frac{\ln \qty(1/\delta)}{2 \NData}}.
    \end{split}
  \end{equation}
\end{theorem}
From this theorem, we can easily derive an upper bound for the excess risk of the empirical risk minimizer as follows.
\begin{corollary}
  \label{cor:ERMBound}
  Define the empirical risk minimizer $\hat{\Hypothesis} \in \HypothesisSp$ and expected loss minimizer by $\Hypothesis^{*} \in \HypothesisSp$ by
  \begin{equation}
    \begin{split}
      \hat{\Hypothesis}
      \DefEq
      \ArgMin_{\Hypothesis \in \HypothesisSp} \hat{\Risk}_{\DataSet} \qty(\Hypothesis),
      \quad
      \Hypothesis^{*}
      \DefEq
      \ArgMin_{\Hypothesis \in \HypothesisSp} \Risk \qty(\Hypothesis),
    \end{split}
  \end{equation}
  and we call $\Risk \qty(\hat{\Hypothesis}) - \Risk \qty(\Hypothesis^{*})$ the excess risk of $\hat{\Hypothesis}$.
  Then for any $\delta \in \Real_{>0}$ and with probability at least $1 - \delta$ we have that
  \begin{equation}
    \begin{split}
      \Risk \qty(\hat{\Hypothesis}) - \Risk \qty(\Hypothesis^{*})
      \le
      2 \RdmCmpl_{\NData} \qty(\HypothesisSp) + 2 \LossBound_{\Loss} \sqrt{\frac{\ln \qty(2/\delta)}{2 \NData}}.
    \end{split}
  \end{equation}
\end{corollary}
\begin{proof}
  We have that
  \begin{equation}
    \begin{split}
      & \Risk \qty(\hat{\Hypothesis}) - \Risk \qty(\Hypothesis^{*})
      \\
      & =
      \qty(\Risk \qty(\hat{\Hypothesis}) - \hat{\Risk}_{\DataSet} \qty(\hat{\Hypothesis})) + \qty(\hat{\Risk}_{\DataSet} \qty(\hat{\Hypothesis}) - \hat{\Risk}_{\DataSet} \qty(\Hypothesis^{*}))
      \\
      & \quad + \qty(\hat{\Risk}_{\DataSet} \qty(\Hypothesis^{*}) - \Risk \qty(\Hypothesis^{*}))
      \\
      & \le
      \qty(\Risk \qty(\hat{\Hypothesis}) - \hat{\Risk}_{\DataSet} \qty(\hat{\Hypothesis})) + \qty(\hat{\Risk}_{\DataSet} \qty(\Hypothesis^{*}) - \Risk \qty(\Hypothesis^{*})),
    \end{split}
  \end{equation}
  where the last inequality holds from the definition of $\hat{\Hypothesis}$. We complete the proof by evaluating the first and second term by \Thm \ref{thm:ULLN} and Hoeffding's inequality, respectively.
\end{proof}
Therefore, it suffices to obtain the upper bound of HOE model's Rademacher complexity, which is given below.
Let $\BoundedSet \subset \Loid^{\NAxes}$. We define a hypothesis function class $\Hypothesis \qty(\cdot; \BoundedSet)$ by
\begin{equation}
  \Hypothesis \qty(\cdot; \BoundedSet)
  \DefEq
  \qty{\Hypothesis \qty(\cdot; \qty(\ReprVec_{\IEntity})_{\IEntity=1}^{\NEntities}) \middle| \qty(\ReprVec_{\IEntity})_{\IEntity=1}^{\NEntities} \in \BoundedSet}.
\end{equation}
We evaluate the Rademacher complexity of the hypothesis function class $\Hypothesis \qty(\cdot; \BoundedSet^{\CBound})$ defined by
\begin{equation}
  \begin{split}
    & \RdmCmpl_{\NCmps} \qty(\Hypothesis \qty(\cdot; \BoundedSet^{\CBound}))
    \\
    & \DefEq \Expect_{\qty(\EntityI, \EntityII, \EntityIII)} \Expect_{\RdmVec} \qty[\sup_{\qty(\ReprVec_{\IEntity})_{\IEntity=1}^{\NEntities} \in \BoundedSet^{\CBound}} \frac{1}{\NCmps} \sum_{\ICmp=1}^{\NCmps} \Rdm_{\ICmp} \Hypothesis \qty(\EntityI_{\ICmp}, \EntityII_{\ICmp}, \EntityIII_{\ICmp}; \qty(\ReprVec_{\IEntity})_{\IEntity=1}^{\NEntities})].
  \end{split}
\end{equation}
The following evaluates the complexity.
\begin{lemma}
  \label{lem:HOEComplexity}
  The Rademacher complexity of the hypothesis function class $\Hypothesis \qty(\cdot; \BoundedSet^{\CBound})$ satisfies the following inequality:
  \begin{equation}
    \begin{split}
      & \RdmCmpl_{\NCmps} \qty(\Hypothesis \qty(\cdot; \BoundedSet^{\CBound}))
      \\
      & \le \qty(\cosh^{2} \CBound + \sinh^{2} \CBound) \qty(\sqrt{\frac{2 \NEntities \ln \NEntities}{\NCmps}} + \frac{\NEntities \ln \NEntities}{6 \NCmps}).
    \end{split}
  \end{equation}
\end{lemma}

\begin{proof}
Define $\GramSet_{\Radius}^{-}, \GramSet_{\Radius}^{+}, \overline{\GramSet}_{\Radius}^{-}, \overline{\GramSet}_{\Radius}^{+} \subset \Sym^{\NEntities, \NEntities}$ by 
\begin{equation}
  \begin{split}
    \GramSet^{\CBound} 
    & \DefEq 
    \qty{\LGramMat^{+} - \LGramMat^{-}\middle| \textrm{$\LGramMat^{-}, \LGramMat^{+} \in \Sym^{\NEntities, \NEntities} $ satisfy \textbf{(a)-(d)}, \textbf{(f)}.}},
    \\
    \overline{\GramSet}^{\CBound} 
    & \DefEq 
    \qty{\LGramMat^{+} - \LGramMat^{-}\middle| \textrm{$\LGramMat^{-} \in \NucGramSpace^{\NEntities \cosh^{2} \CBound}, \LGramMat^{+} \in \NucGramSpace^{\NEntities \sinh^{2} \CBound}$.}},
  \end{split}
\end{equation}
where \textbf{(a)-(f)} are the conditions defined in \Lem \ref{lem:GramianDecomposition}, and for $\NucBound \in \Real_{\ge 0}$, $\NucGramSpace^{\NucBound}$ is defined by
\begin{equation}
  \begin{split}
    \NucGramSpace^{\NucBound} 
    & \DefEq 
    \qty{\LGramMat \in \Sym^{\NEntities, \NEntities} \middle| \textrm{$\LGramMat \PosSemiDef \ZeroMat$, $\norm{\LGramMat}_{*} \le \NucBound$.}}.
  \end{split}
\end{equation}
Note that $\GramSet_{\Radius}^{\CBound} \subset \overline{\GramSet}^{\CBound}$. Let $\Rdm_{1}, \Rdm_{2}, \dots, \Rdm_{\NCmps}$ be i.i.d Rademacher random variables and $\RdmVec = \mqty[\Rdm_{1} & \Rdm_{2} & \dots & \Rdm_{\NCmps}]^\Transpose$.
The Rademacher complexity is calculated as follows:
\begin{equation}
  \begin{split}
    & \RdmCmpl_{\NCmps} \qty(\Hypothesis \qty(\cdot; \BoundedSet^{\CBound}))
    \\
    & \DefEq \Expect_{\qty(\EntityI, \EntityII, \EntityIII)} \Expect_{\RdmVec} \qty[\sup_{\qty(\ReprVec_{\IEntity})_{\IEntity=1}^{\NEntities} \in \BoundedSet^{\CBound}} \frac{1}{\NCmps} \sum_{\ICmp=1}^{\NCmps} \Rdm_{\ICmp} \Hypothesis \qty(\EntityI_{\ICmp}, \EntityII_{\ICmp}, \EntityIII_{\ICmp}; \qty(\ReprVec_{\IEntity})_{\IEntity=1}^{\NEntities})]
    \\
    & = \Expect_{\qty(\EntityI, \EntityII, \EntityIII)} \Expect_{\RdmVec} \qty[\sup_{\LGramMat \in \GramSet^{\CBound}} \frac{1}{\NCmps} \sum_{\ICmp=1}^{\NCmps} \Rdm_{\ICmp} \FInProd{\LGramMat}{\CmpMat_{\EntityI_{\ICmp}, \EntityII_{\ICmp}, \EntityIII_{\ICmp}}}]
    \\
    & \le \Expect_{\qty(\EntityI, \EntityII, \EntityIII)} \Expect_{\RdmVec} \qty[\sup_{\LGramMat \in \overline{\GramSet}^{\CBound}} \frac{1}{\NCmps} \sum_{\ICmp=1}^{\NCmps} \Rdm_{\ICmp} \FInProd{\LGramMat}{\CmpMat_{\EntityI_{\ICmp}, \EntityII_{\ICmp}, \EntityIII_{\ICmp}}}],
  \end{split}
\end{equation}
where the last inequality results from $\GramSet_{\Radius}^{\CBound} \subset \overline{\GramSet}^{\CBound}$.
We can decompose the integrand of the above expectation operator as follows:
\begin{equation}
  \begin{split}
    & \sup_{\LGramMat \in \overline{\GramSet}^{\CBound}} \frac{1}{\NCmps} \sum_{\ICmp=1}^{\NCmps} \Rdm_{\ICmp} \FInProd{\LGramMat}{\CmpMat_{\EntityI_{\ICmp}, \EntityII_{\ICmp}, \EntityIII_{\ICmp}}}
    \\
    & = \sup_{\substack{\LGramMat^{-} \in \NucGramSpace^{\NEntities \sinh^{2} \CBound}, \\ \LGramMat^{+} \in \NucGramSpace^{\NEntities \cosh^{2} \CBound}}} \frac{1}{\NCmps} \sum_{\ICmp=1}^{\NCmps} \Rdm_{\ICmp} \FInProd{\LGramMat^{+} - \LGramMat^{-}}{\CmpMat_{\EntityI_{\ICmp}, \EntityII_{\ICmp}, \EntityIII_{\ICmp}}}
    \\
    & \le \sup_{\LGramMat^{+} \in \NucGramSpace^{\NEntities \sinh^{2} \CBound}} \frac{1}{\NCmps} \sum_{\ICmp=1}^{\NCmps} \Rdm_{\ICmp} \FInProd{\LGramMat^{+}}{\CmpMat_{\EntityI_{\ICmp}, \EntityII_{\ICmp}, \EntityIII_{\ICmp}}}
    \\
    & \quad + \sup_{\LGramMat^{-} \in \NucGramSpace^{\NEntities \cosh^{2} \CBound}} \frac{1}{\NCmps} \sum_{\ICmp=1}^{\NCmps} \Rdm_{\ICmp} \FInProd{\LGramMat^{-}}{\CmpMat_{\EntityI_{\ICmp}, \EntityII_{\ICmp}, \EntityIII_{\ICmp}}}.
  \end{split}
\end{equation}
Hence, we have 
\begin{equation}
  \begin{split}
    & \RdmCmpl_{\NCmps} \qty(\Hypothesis \qty(\cdot; \BoundedSet^{\CBound}))
    \\
    & \le \RdmCmpl_{\NCmps}^{\GramSymb} \qty(\FInProd{\NucGramSpace^{\NEntities \sinh^{2} \CBound}}{\cdot}) + \RdmCmpl_{\NCmps}^{\GramSymb} \qty(\FInProd{\NucGramSpace^{\NEntities \cosh^{2} \CBound}}{\cdot}),
  \end{split}
\end{equation}
where $\RdmCmpl_{\NCmps}^{\GramSymb} \qty(\FInProd{\NucGramSpace^{\NucBound}}{\cdot})$ is defined as
\begin{equation}
  \Expect_{\qty(\EntityI, \EntityII, \EntityIII)} \Expect_{\RdmVec} \qty[\sup_{\LGramMat \in \NucGramSpace^{\NucBound}} \frac{1}{\NCmps} \sum_{\ICmp=1}^{\NCmps} \Rdm_{\ICmp} \FInProd{\LGramMat}{\CmpMat_{\EntityI_{\ICmp}, \EntityII_{\ICmp}, \EntityIII_{\ICmp}}}].
\end{equation}
We can bound $\RdmCmpl_{\NCmps}^{\GramSymb} \qty(\FInProd{\NucGramSpace^{\NucBound}}{\cdot})$ by the following lemma.
\begin{lemma}
\label{lem:ComplexityBound}
\begin{equation}
  \begin{split}
  \RdmCmpl_{\NCmps}^{\GramSymb} \qty(\FInProd{\NucGramSpace^{\NucBound}}{\cdot}) \le \frac{\NucBound}{\NEntities} \qty(\sqrt{\frac{2 \qty(\NEntities + 1) \ln \NEntities}{\NCmps}} + \cdot \frac{\NEntities \ln \NEntities}{\sqrt{12}\NCmps}).
  \end{split}
\end{equation}
\end{lemma}
We prove \Lem \ref{lem:ComplexityBound} later.
By \Lem \ref{lem:ComplexityBound}, we obtain
\begin{equation}
  \begin{split}
    & \RdmCmpl_{\NCmps} \qty(\Hypothesis \qty(\cdot; \BoundedSet^{\CBound}))
    \\
    & \le \qty(\cosh^{2} \CBound + \sinh^{2} \CBound) \qty(\sqrt{\frac{2 \qty(\NEntities + 1) \ln \NEntities}{\NCmps}} + \cdot \frac{\NEntities \ln \NEntities}{\sqrt{12}\NCmps}),
  \end{split}
\end{equation}
which completes the proof.
\end{proof}

\begin{proof}[Proof of \Lem \ref{lem:ComplexityBound}]
For $\NucBound \in \Real_{\ge 0}$, define $\NucGramSpace_{1}^{\NucBound}$ by
\begin{equation}
  \qty{\NucBound \Vec{u} \Vec{u}^\Transpose \middle| \Vec{u} \in \Real^{\NEntities}, \norm{\Vec{u}}_{2} = 1.}.
\end{equation}
Since $\NucGramSpace_{1}^{\NucBound}$ is the convex hull of $\NucGramSpace^{\NucBound}$ and the Rademacher complexity of a function class equals that of its convex hull \citep[\Thm 12-2]{DBLP:journals/jmlr/BartlettM02}, we have that
\begin{equation}
  \begin{split}
    \label{eqn:HOERademacher}
    & \RdmCmpl_{\NCmps}^{\GramSymb} \qty(\FInProd{\NucGramSpace^{\NucBound}}{\cdot})
    \\
    & =
    \RdmCmpl_{\NCmps}^{\GramSymb} \qty(\FInProd{\NucGramSpace_{1}^{\NucBound}}{\cdot})
    \\
    & = \Expect_{\qty(\EntityI, \EntityII, \EntityIII)} \Expect_{\RdmVec} \qty[\sup_{\norm{\Vec{u}}_{2} \le 1} \frac{1}{\NCmps} \sum_{\ICmp=1}^{\NCmps} \Rdm_{\ICmp} \FInProd{\NucBound \Vec{u} \Vec{u}^\Transpose}{\CmpMat_{\EntityI_{\ICmp}, \EntityII_{\ICmp}, \EntityIII_{\ICmp}}}]
    \\
    & = \Expect_{\qty(\EntityI, \EntityII, \EntityIII)} \Expect_{\RdmVec} \qty[\sup_{\norm{\Vec{u}}_{2} \le 1} \frac{1}{\NCmps} \sum_{\ICmp=1}^{\NCmps} \Rdm_{\ICmp} \NucBound \Tr(\Vec{u} \Vec{u}^\Transpose\CmpMat_{\EntityI_{\ICmp}, \EntityII_{\ICmp}, \EntityIII_{\ICmp}})]
    \\
    & = \frac{\NucBound}{\NCmps} \Expect_{\qty(\EntityI, \EntityII, \EntityIII)} \Expect_{\RdmVec} \qty[\sup_{\norm{\Vec{u}}_{2} \le 1} \sum_{\ICmp=1}^{\NCmps} \Rdm_{\ICmp} \Tr(\Vec{u}^\Transpose \CmpMat_{\EntityI_{\ICmp}, \EntityII_{\ICmp}, \EntityIII_{\ICmp}} \Vec{u})]
    \\
    & = \frac{\NucBound}{\NCmps} \Expect_{\qty(\EntityI, \EntityII, \EntityIII)} \Expect_{\RdmVec} \qty[\sup_{\norm{\Vec{u}}_{2} \le 1} \Tr(\Vec{u}^\Transpose \qty(\sum_{\ICmp=1}^{\NCmps} \Rdm_{\ICmp} \CmpMat_{\EntityI_{\ICmp}, \EntityII_{\ICmp}, \EntityIII_{\ICmp}}) \Vec{u})]
    \\
    & = \frac{\NucBound}{\NCmps} \Expect_{\qty(\EntityI, \EntityII, \EntityIII)} \Expect_{\RdmVec} \qty[\norm{\sum_{\ICmp=1}^{\NCmps} \Rdm_{\ICmp} \CmpMat_{\EntityI_{\ICmp}, \EntityII_{\ICmp}, \EntityIII_{\ICmp}}}_{\OpSymb, 2}],
  \end{split}
\end{equation}
where $\norm{\cdot}_{\OpSymb, 2}$ denotes the operator norm with respect to the 2-norm defined by
\begin{equation}
  \norm{\Mat{A}}_{\OpSymb, 2} \DefEq \max_{\norm{\Vec{u}}_{2} \le 1} \norm{\Mat{A} \Vec{u}}_{2}
\end{equation}
To evaluate this, we can apply the following the matrix Bernstein inequality.
\begin{theorem}[\citep{tropp2015introduction} \Thm 6.6.1]
  \label{thm:MatrixBernstein}
  Let $\Mat{A}_{1}, \Mat{A}_{2}, \dots, \Mat{A}_{\NCmps} \in \Sym^{\NEntities, \NEntities}$ be independent random matrices that satisfies
  \begin{equation}
    \Expect \Mat{A}_{\ICmp} = \ZeroMat,
    \quad
    \norm{\Mat{A}_{\ICmp}}_{\OpSymb, 2} \le \sigma.
  \end{equation}
  Then
  \begin{equation}
    \Expect \norm{\sum_{\ICmp=1}^{\NCmps} \Mat{A}_{\ICmp}}_{\OpSymb, 2} \le \sqrt{2 \MatVar \qty(\sum_{\ICmp=1}^{\NCmps} \Mat{A}_{\ICmp}) \ln \NEntities} + \frac{1}{3} \sigma \ln \NEntities,
  \end{equation}
  where $\MatVar$ is the matrix variance statistics defined by
  \begin{equation}
    \MatVar \qty(\Mat{A}) \DefEq \norm{\Expect \Mat{A}^{2}}_{\OpSymb, 2}.
  \end{equation}
\end{theorem}
Note that $\MatVar \qty(\sum_{\ICmp=1}^{\NCmps} \Mat{A}_{\ICmp}) = \norm{\sum_{\ICmp=1}^{\NCmps} \Expect \Mat{A}_{\ICmp}^{2}}_{\OpSymb, 2}$ is valid since $\Mat{A}_{1}, \Mat{A}_{2}, \dots, \Mat{A}_{\NCmps} \in \Sym^{\NEntities, \NEntities}$ are independent.
We apply \Thm \ref{thm:MatrixBernstein} to the right hand side of \eqref{eqn:HOERademacher} by substituting $\Mat{A}_{\ICmp}$ by $\Rdm_{\ICmp} \CmpMat_{\EntityI_{\ICmp}, \EntityII_{\ICmp}, \EntityIII_{\ICmp}}$. Here, $\Expect \Rdm_{\ICmp} \CmpMat_{\EntityI_{\ICmp}, \EntityII_{\ICmp}, \EntityIII_{\ICmp}} = \ZeroMat$ is valid because $\Expect \Rdm_{\ICmp} = 0$. The singular values of $\Rdm_{\ICmp} \CmpMat_{\EntityI_{\ICmp}, \EntityII_{\ICmp}, \EntityIII_{\ICmp}}$ is equal to those of $\Rdm_{\ICmp} \tilde{\CmpMat}$, where
\begin{equation}
  \tilde{\CmpMat} \DefEq
  \mqty[
    0 & - \frac{1}{2} & + \frac{1}{2} \\ 
    - \frac{1}{2} & 0 & 0 \\ 
    + \frac{1}{2} & 0 & 0 \\ 
  ].
\end{equation}
Since 
\begin{equation}
  \qty(\Rdm_{\ICmp} \tilde{\CmpMat})^2 = \qty(\Rdm_{\ICmp} \tilde{\CmpMat})^\Transpose \qty(\Rdm_{\ICmp} \tilde{\CmpMat})
  =
  \mqty[
    + \frac{1}{2} & 0 & 0 \\ 
    0 & + \frac{1}{4} & - \frac{1}{4} \\ 
    0 & - \frac{1}{4} & + \frac{1}{4} \\ 
  ],
\end{equation}
and its eigenvalues are $0, + \frac{1}{2}, + \frac{1}{2}$, the singular values of $\tilde{\CmpMat}$ are $0, + \frac{1}{\sqrt{2}}, + \frac{1}{\sqrt{2}}$.
Hence, we have that $\norm{\Rdm_{\ICmp} \CmpMat_{\EntityI_{\ICmp}, \EntityII_{\ICmp}, \EntityIII_{\ICmp}}}_{\OpSymb, 2} \le \frac{1}{\sqrt{2}}$.
Lastly, we evaluate $\norm{\sum_{\ICmp=1}^{\NCmps} \Expect \qty(\Rdm_{\ICmp} \CmpMat_{\EntityI_{\ICmp}, \EntityII_{\ICmp}, \EntityIII_{\ICmp}})^2}_{\OpSymb, 2} = \NCmps \norm{\Expect \CmpMat_{\EntityI_{\ICmp}, \EntityII_{\ICmp}, \EntityIII_{\ICmp}}^{2}}_{2}$. The diagonal elements and off-diagonal elements of $\Expect \CmpMat_{\EntityI_{\ICmp}, \EntityII_{\ICmp}, \EntityIII_{\ICmp}}^{2}$ are all $\frac{1}{\NEntities}$ and all $- \frac{1}{2} \frac{1}{\NEntities (\NEntities - 1)}$, respectively, because the sum of the diagonal elements and that of the off-diagonal elements in $\CmpMat_{\EntityI_{\ICmp}, \EntityII_{\ICmp}, \EntityIII_{\ICmp}}$ are always $1$ and $- \frac{1}{2}$, respectively, and from the symmetricity among the $\NEntities$ diagonal elements and $\NEntities \qty(\NEntities - 1)$ off-diagonal elements. Hence, we have $\Expect \CmpMat_{\EntityI_{\ICmp}, \EntityII_{\ICmp}, \EntityIII_{\ICmp}}^{2} = \qty(\frac{1}{\NEntities} + \frac{1}{2} \frac{1}{\NEntities (\NEntities - 1)}) \Identity - \frac{1}{2} \frac{1}{\NEntities (\NEntities - 1)} \Vec{1} \Vec{1}^\Transpose$. whose eigenvalues (singular values) are $\frac{1}{\NEntities} + \frac{1}{2} \frac{1}{\NEntities (\NEntities - 1)}$ (multiplicity $\NEntities - 1$) and $\frac{1}{\NEntities}$ (multiplicity $1$). Thus, we have
\begin{equation}
  \begin{split}
    \MatVar \qty(\sum_{\ICmp=1}^{\NCmps} \Rdm_{\ICmp} \CmpMat_{\EntityI_{\ICmp}, \EntityII_{\ICmp}, \EntityIII_{\ICmp}}) 
    & = \NCmps \norm{\Expect \CmpMat_{\EntityI_{\ICmp}, \EntityII_{\ICmp}, \EntityIII_{\ICmp}}^{2}}_{2}
    \\
    & = \NCmps \qty(\frac{1}{\NEntities} + \frac{1}{2} \frac{1}{\NEntities (\NEntities - 1)})
    \\
    & = \frac{\NCmps}{2} \qty(\frac{1}{\NEntities-1} + \frac{1}{\NEntities})
    \\
    & = \frac{\NCmps}{2} \frac{1}{\NEntities^2} \qty(\frac{\NEntities^2}{\NEntities-1} + \NEntities)
    \\
    & \le \frac{\NCmps}{\NEntities^2} \qty(\NEntities + 1)
  \end{split}
\end{equation}
.
Therefore we have
\begin{equation}
  \begin{split}
    & \Expect_{\qty(\EntityI, \EntityII, \EntityIII)} \Expect_{\RdmVec} \qty[\norm{\sum_{\ICmp=1}^{\NCmps} \Rdm_{\ICmp} \CmpMat_{\EntityI_{\ICmp}, \EntityII_{\ICmp}, \EntityIII_{\ICmp}}}_{\OpSymb, 2}]
    \\
    & \le \frac{1}{\NEntities} \sqrt{2 \NCmps \qty(\NEntities + 1) \ln \NEntities} + \frac{1}{3} \cdot \frac{1}{\sqrt{2}} \ln \NEntities,
  \end{split}
\end{equation}
which completes the proof.
\end{proof}

\begin{proof}[Proof of \Thm \ref{thm:HOEBound}]
We complete the proof by applying \Cor \ref{cor:ERMBound} to \Lem \ref{lem:HOEComplexity}. 
\end{proof}

\end{document}